\documentclass[10pt,journal,compsoc]{IEEEtran}
\usepackage[table]{xcolor}
\usepackage{xcolor,colortbl,multicol,multirow,mathrsfs,epsfig,graphicx,subfigure,amsmath,amsthm,amsfonts,amssymb,bm,algorithmic,algorithm}
\usepackage{caption2}
\usepackage{colortbl}
 \usepackage{setspace}
\newtheorem{definition}{Definition}
\allowdisplaybreaks[1]
\definecolor{mygray}{gray}{.9}
\definecolor{mypink}{rgb}{.99,.91,.95}
\definecolor{mycyan}{cmyk}{.3,0,0,0}
\newtheorem{proposition}{Proposition}

\renewcommand{\algorithmicrequire}{\textbf{Input:}}  
\renewcommand{\algorithmicensure}{\textbf{Procedure:}} 
\renewcommand{\algorithmiclastcon}{\textbf{Output:}}

\begin{document}
\title{A General Multi-Graph Matching Approach via Graduated Consistency-regularized Boosting}
\author{Junchi Yan, Minsu Cho, Hongyuan Zha, Xiaokang Yang~\IEEEmembership{Senior Member,~IEEE} and Stephen M. Chu
\thanks{A preliminary version of the presented paper appeared in \cite{YanECCV14}. The current paper makes several extensions and improvements: i) propose a new graduated consistency-regularized affinity score boosting algorithm (ISB-GC in Alg.\ref{alg:ISB-GC}) which in general achieves more accurate matching results, meanwhile helps reinterpret the method (ISB-GC-U in  Alg.\ref{alg:ISB-EGC}) of the conference paper; ii) improve the robustness of the methods in face of outliers by integrating a node-wise consistency/affinity-driven inlier eliciting mechanism; iii) present more technical details of the proposed algorithms and convergence discussion which are not fully described in \cite{YanECCV14}; iv) perform extensive evaluations that involve more various settings and additional datasets. In addition, more emerging state-of-the-arts \cite{ChenICML14,ChoCVPR14,YanTIP15} are compared especially after the year 2013. The source code of our methods will be made public available.}
\IEEEcompsocitemizethanks{\IEEEcompsocthanksitem{J. Yan is with Shanghai Key Laboratory of Media Processing and Transmission, Department of Electronic Engineering, Shanghai Jiao Tong University, Shanghai, 200240 China.
He is secondarily affiliated with IBM Research -- China as a Research Staff Member. e-mail: yanjunchi@sjtu.edu.cn}
\IEEEcompsocthanksitem{M. Cho is with WILLOW team of INRIA/Ecole Normale Superieure, Paris, France. e-mail: minsu.cho@inria.fr}
\IEEEcompsocthanksitem{H. Zha is with Software Engineering Institute, East China Normal University, Shanghai, 200062 China, and School of Computational Science and Engineering, College of Computing, Georgia Institute of Technology, Atlanta, Georgia, 30332, USA. e-mail: zha@cc.gatech.edu}
\IEEEcompsocthanksitem{X. Yang is with Shanghai Key Laboratory of Media Processing and Transmission, Department of Electronic Engineering, Shanghai Jiao Tong University, Shanghai, 200240 China. e-mail: xkyang@sjtu.edu.cn}
\IEEEcompsocthanksitem{S. Chu is with IBM T.J. Watson Research Center. e-mail: schu@us.ibm.com}
}
}

\markboth{Junchi Yan et al. A General Multi-Graph Matching Approach via Graduated Consistency-regularized Boosting}%
{Shell \MakeLowercase{\textit{et al.}}: Bare Demo of IEEEtran.cls for Journals}
\maketitle

\begin{abstract}
This paper addresses the problem of matching $N$ weighted graphs referring to an identical object or category. More specifically, matching the common node correspondences among graphs. This multi-graph matching problem involves two ingredients affecting the overall accuracy: i) the local pairwise matching affinity score among graphs; ii) the global matching consistency that measures the uniqueness of the pairwise matching results by different chaining orders. Previous studies typically either enforce the matching consistency constraints in the beginning of iterative optimization, which may propagate matching error both over iterations and across graph pairs; or separate affinity optimizing and consistency regularization in two steps. This paper is motivated by the observation that matching consistency can serve as a regularizer in the affinity objective function when the function is biased due to noises or inappropriate modeling. We propose multi-graph matching methods to incorporate the two aspects by boosting the affinity score, meanwhile gradually infusing the consistency as a regularizer. Furthermore, we propose a node-wise consistency/affinity-driven mechanism to elicit the common inlier nodes out of the irrelevant outliers. Extensive results on both synthetic and public image datasets demonstrate the competency of the proposed algorithms.
\end{abstract}
\section{Introduction}
\IEEEPARstart{G}{raph} matching (GM) \cite{ConteIJPRAI04,FoggiaIJPRAI14} has received extensive attentions over decades, and has found wide applications in various problems such as bioinformatics \cite{ZaslavskiyBio09}, data fusion \cite{WilliamsPRL97}, graphics \cite{HuangSGP13} and information retrieval \cite{LiuVCG14}, name a few. In particular, GM lies at the heart of a range computer vision tasks as diverse as object recognition, shape matching, object tracking, and image labeling among others, which require finding visual correspondences -- refer to \cite{YanTIP15} for more specific references. Different from the point based matching or registration methods such as RANSAC \cite{FischlerCACM81} and Iterative Closet Point (ICP)~\cite{ZhangIJCV94}, GM incorporates both the unary {\em node-to-node}, and the second-order {\em edge-to-edge} structural similarity. By encoding the geometrical information in the representation and matching process, GM methods are in general supposed to be more robust for solving the correspondence problems. Due to its well-known NP-complete nature, existing GM methods involve either finding approximate solutions \cite{ZhouCVPR12,ChoECCV10,LeordeanuNIPS09} or obtaining the global optima in polynomial time for a few types of graphs, including the planar graph \cite{EppsteinSODA95}, bounded valence graph \cite{luks1982isomorphism} and tree structure \cite{aho1989code}.

Most GM methods focus on establishing one-to-one correspondences (which is also the interest of this paper) between a pair of feature points \cite{LeordeanuIJCV12,GoldPAMI96,ChoECCV10,CaetanoPAMI09,EgoziPAMI13}. In general, the pairwise GM problem can be formulated as a quadratic assignment problem (QAP) that accounts for both individual node matches (unary terms) and pairs of matches (pairwise terms). The QAP formulations can be roughly divided into two sub-categories \cite{ZhouCVPR12}: i) the Koopmans-Beckmann's QAP \cite{KBQAP57} $\text{tr}(\textbf{X}^T\textbf{A}_1\textbf{X}\textbf{A}_2)+\text{tr}(\textbf{K}_p^T\textbf{X})$ where $\textbf{X}$ refers to the assignment matrix between two graphs as will be introduced later in the paper. $\textbf{A}_1$, $\textbf{A}_2$ are the weighted adjacency matrices and $\textbf{K}_p$ is the node-to-node similarity matrix between two graphs; ii) the more general Lawler's QAP problem \cite{LawlerMS63} by $\text{vec}(\textbf{X})^T\textbf{K}\text{vec}(\textbf{X})$ where $\textbf{K}$ is the affinity matrix encoding unary and second-order edge affinities. The former can always be represented as a special case of the Lawler's QAP by setting $\textbf{K}=\textbf{A}_2\otimes\textbf{A}_1$.

Nevertheless, in many applications, visual objects do not appear in isolation or in pair, but more frequently in collections or families. Such collections provide a context and potentially allow for higher quality matching and analysis. Given a batch of graphs referring to an identical or related structure, it is required to find the global matchings across all graphs. Due to the advance of modern imaging and scanning technologies, multi-graph matching is applied to infusing multi-source sensor data \cite{WilliamsPRL97}. Graphic shape analysis and search often require to model objects by multi-view assembly~\cite{FunkhouserTOG04}. And \cite{LiuVCG14} applies multi-graph matching to the problem of multi-source topic alignment. It is also related to graph clustering, classification and indexing. Several multi-graph matching methods \cite{RibaltaCVIU11,HuangTOG12,PachauriNIPS13,YanICCV13,RibaltaIJPRAI13,YanECCV14,ChenICML14} are proposed, which will be detailed in the next section.
\section{Related Work and Models}
\textbf{Pairwise graph matching} The problem of pairwise GM involving two graphs has been extensively studied in the literature. It is beyond the scope of this paper for an exhaustive summary of these work. As will be detailed later in this paper, our methods and many other multi-graph matching methods \cite{YanICCV13,YanTIP15,ChenICML14,PachauriNIPS13} build on top of pairwise matching solvers as a black-box, regardless the specific context in which these pairwise GM techniques are devised, such as machine learning based pairwise GM approaches \cite{LeordeanuIJCV12,CaetanoPAMI09,LeordeanuICCV11,ChoICCV13,HuCVPR13}, hyper-graph matching approaches \cite{LeordeanuICCV11,ChertokPAMI10,DuchenneCVPR09,LeeCVPR11,ZassCVPR08} and the recent work \cite{ChoCVPR14} tailored for massive outliers.

Most pairwise GM methods \cite{GoldJANN96,LeordeanuICCV05,CourNIPS06,LeordeanuNIPS09,ChoECCV10,EgoziPAMI13} are based on the Lawler's QAP. Given two graphs $\mathcal{G}_1$ and $\mathcal{G}_2$ of node size $n_1$ and $n_2$, an affinity matrix $\textbf{K}\in\mathbb{R}^{n_1n_2\times n_1n_2}$ is defined such that its elements $\{K_{ia;jb}\}_{i,j=1}^{n_1}{_{a,b=1}^{n_2}}$ measure the edge pair affinity $\{(v_i,v_j)\leftrightarrow (v_a,v_b)\}{_{i,j=1}^{n_1}}{_{a,b=1}^{n_2}}$ from two graphs. The diagonal term $\{K_{ia;ia}\}_{i=1,a=1}^{n_1,n_2}$ describes the unary affinity of a node match $\{v_i\leftrightarrow v_a\}_{i=1,a=1}^{n_1,n_2}$.
More rigorously, most existing graph matching work \cite{YanTIP15,ChoECCV10,ZhouCVPR12,YanECCV14} follow a convention for setting $\textbf{K}$ such that the element $K_{ia;jb}$ for the edge pair $(v_i,v_j)\leftrightarrow(v_a,v_b)$ is located at the (($a$-1)$n_1$+$i$)-th row and (($b$-1)$n_2$+$j$)-th column of $\textbf{K}$.

We define the assignment matrix $\textbf{X}\in \{0,1\}^{n_1\times n_2}$ for two graphs which establishes the one-to-one node correspondence such that $X_{ia}=1$ if node $v_i$ matches node $v_a$ ($0$ otherwise). The problem of GM involves finding the optimal correspondence $\textbf{X}$ such that the sum of the node and edge compatibility between two graphs is maximized. Without loss of generality, similar to \cite{ChoECCV10,ZhouCVPR12}, by assuming $n_1\geq n_2$ for different sizes of graphs, it leads to the following widely used two-way constrained quadratic assignment problem (QAP) (\emph{e.g.} \cite{ChoECCV10,ZhouCVPR12} and the references therein):
\begin{small}
\begin{align}\label{eq:pairwise_objective}
&\textbf{X}^* = \arg \max_\textbf{X} \text{vec}(\textbf{X})^T\textbf{K}\text{vec}(\textbf{X})\\\notag
s.t.& \quad \textbf{X}\textbf{1}_{n_2}\leq \textbf{1}_{n_1}\quad \textbf{1}^T_{n_1}\textbf{X}=\textbf{1}^T_{n_2} \quad \textbf{X} \in \{0,1\}^{n_1\times n_2}
\end{align}
\end{small}
The constraints refer to the two-way one-to-one node mapping: a node from graph $\mathcal{G}_1$ can match at most one node in $\mathcal{G}_2$ and every node in $\mathcal{G}_2$ is corresponding to one node in $\mathcal{G}_1$. There is no (one/many)-to-many matching.

The above formulation is as general as it allows two graphs having unequal number of nodes ($n_1\neq n_2$). A common protocol adopted by the previous studies \cite{GoldJANN96,ZhouCVPR12,YanTIP15} is converting $\textbf{X}$ from an assignment matrix ($\textbf{X}\textbf{1}_{n_2}\leq \textbf{1}_{n_1}$) to a permutation matrix ($\textbf{X}\textbf{1}_{n_2}=\textbf{1}_{n_1}$) by adding dummy nodes to one graph (\emph{i.e.} adding slack variables to the assignment matrix and augment the affinity matrix by zeros) in case $n_1\neq n_2$. This is a common technique from linear programming and is widely adopted by \cite{GoldJANN96,ZhouCVPR12} such that is supposed can handle superfluous nodes in a statistically robust manner (see the last paragraph of Sec.2.3 in \cite{GoldJANN96}). By taking this step, the graphs are of equal sizes. This preprocessing also opens up the applicability of existing multi-graph methods \cite{RibaltaIJPRAI13,PachauriNIPS13,RibaltaCIARP09} as they all assume that all graphs are of equal sizes.
Therefore, the following formulation is derived which assumes $n_1=n_2=n$.
\begin{small}
\begin{align}\label{eq:overall_objective}
&\textbf{X}_{12}^{*} = \arg \max_{\textbf{X}_{12}} \text{vec}(\textbf{X}_{12})^T\textbf{K}_{12}\text{vec}(\textbf{X}_{12})\\\notag
s.t.&\quad \textbf{X}_{12}\textbf{1}_{n_2}=\textbf{1}_{n_1}\quad \textbf{1}^T_{n_1}\textbf{X}_{12}=\textbf{1}^T_{n_2} \quad \textbf{X}_{12} \in \{0,1\}^{n_1\times n_2}
\end{align}
\end{small}
$\textbf{X}_{12}$ is a permutation matrix constructed by augmenting the raw matching matrix with slack columns if $n_1>n_2$. This is also coherent with the ideal case for the problem setting of this paper, matching the common inliers of equal sizes among graphs in the absence of outliers, or they are dismissed by a certain means.

We also mention several other threads for pairwise GM: i) learning the affinity over graphs \cite{CaetanoPAMI09,LeordeanuICCV11,LeordeanuIJCV12,ChoICCV13,HuCVPR13}; ii) incorporating higher-order information \cite{LeordeanuICCV11,ChertokPAMI10,DuchenneCVPR09,LeeCVPR11,ZassCVPR08}; iii) progressive matching \cite{ChoCVPR12} or combining feature detection and matching in an integrated system \cite{CollinsECCV14}. These methods focus on pairwise matching and are methodologically less relevant to this paper.

\textbf{Multi-graph matching} Recently, matching a collection of related graphs becomes an emerging research topic with researchers from different communities including computer vision \cite{YanICCV13,YanTIP15}, machine learning \cite{PachauriNIPS13,ChenICML14}, pattern recognition \cite{RibaltaCIARP09,RibaltaCVIU11,RibaltaIJPRAI13}, graphics \cite{HuangSGP13} and information retrieval \cite{LiuVCG14}, among others. We divide the state-of-the-arts into two categories concerning how the affinity and matching consistency are explored.

i) \emph{affinity score}-driven approaches \cite{RibaltaIJPRAI13,YanICCV13,YanTIP15,GavrilJA87}: For these methods, usually first a compact set of basis pairwise matching variables are generated which can derive the matching solutions for each pair. Then an objective function regarding the overall pairwise matching affinity score is maximized by different algorithms. Sol\'{e}-Ribalta \emph{et al.} \cite{RibaltaIJPRAI13} extend the algorithm Graduated Assignment (GA) \cite{GoldJANN96} for two-graph case to multi-graph. Each graph is associated with an assignment matrix mapping the nodes to a virtual node set. Then the variable set is updated in a deterministic annealing manner to maximize the overall affinity score. Yan et al. \cite{YanICCV13,YanTIP15} adaptively find a reference graph $\mathcal{G}_r$, which induces a compact basis matching variable set $\{\textbf{X}_{kr}\}_{k=1,\neq r}^{N}$ over $N$ graphs. An alternating optimization framework is devised to update these basis variables in a rotating manner. The early work \cite{GavrilJA87} by Gavril \emph{et al.} selects a set of basis variables from the initial pairwise matching solutions via finding the maximum spanning tree on the super graph related to a maximized overall affinity score. In this sense, their work also falls into this category.

ii) \emph{pairwise matching consistency}-driven approaches \cite{RibaltaCIARP09,RibaltaCVIU11,PachauriNIPS13,ChenICML14}: these methods usually are comprised of two steps. First, a pairwise GM solver is employed to obtain the node mapping between all (or a portion of) pairs of graphs. Then, a spectrum smoothing technique is devised to enforce global matching consistency in the sense that two sequential pairwise matching in different chaining orders shall lead to the same solution. Specifically, in \cite{RibaltaCIARP09} and the journal version \cite{RibaltaCVIU11}, Sol\'{e}-Ribalta \emph{et al.} use a hyper-cube tensor to represent the $N$-node matching likelihood, each from $N$ different graphs. Then a greedy method is used to binarize the final solutions satisfying the one-to-one two way constraints. Pachauri \emph{et al.} \cite{RibaltaCVIU11} employ spectral analysis via eigenvector decomposition on the input pairwise matching solutions, and recover the consistent matching solutions. The success of this method lies on the assumption that the input pairwise matchings is corrupted by Gaussian-Wigner noise which is too ideal in reality. A more recent work proposed by Chen \emph{et al.} \cite{ChenICML14} is suitable when only a part of nodes can find their correspondences in other graphs, which they term as `partial similarity'. They formulate the problem into a tractable convex programming problem solved by the popular first-order Alternating Direction Method of Multipliers (ADMM) method tailored for their semidefinite programming problem.

One observation to the above mentioned approaches is that they either enforce consistency early by using a compact set of basis variables for pairwise matchings \cite{YanICCV13,YanTIP15,RibaltaCVIU11,RibaltaCIARP09}, which runs at the risk of propagating errors across iterations and graphs, or assume the initial pairwise matchings are obtained with corruptions by another method, and focus on improving the overall accuracy from the input pairwise matchings via spectral methods, however no affinity information is used in the procedure \cite{RibaltaCIARP09,RibaltaCVIU11,PachauriNIPS13,ChenICML14}.

In particular, a widely used multi-graph matching formulation \cite{YanICCV13,YanTIP15,RibaltaIJPRAI13} is as follows, under the implicit assumption of invertible pairwise matching relations \emph{i.e.} $\textbf{X}_{ij}=\textbf{X}^T_{ji}$, which meanwhile enforces matching consistency over all pairwise matchings:
\begin{footnotesize}
\begin{gather}\label{eq:constrained_multi_match_obj}
\{\textbf{X}^*_{ij}\}_{i=1,j=i+1}^{N-1,N} = \arg\max_{\textbf{X}_{ij}}\sum_{i=1}^{N-1}\sum_{j=i+1}^{N}\text{vec}(\textbf{X}_{ij})^T\textbf{K}_{ij}\text{vec}(\textbf{X}_{ij})\\\notag
s.t.\quad\textbf{I}^T_{n_i}\textbf{X}_{ij}=\textbf{1}^T_{n_j}, \textbf{X}_{ij}\textbf{I}_{n_j}=\textbf{1}_{n_i}, \textbf{X}_{ij}=\textbf{X}_{ik}\textbf{X}_{kj}\in \{0,1\}^{n_i\times n_j}{_{k=1,\neq i,j}^N}
\end{gather}
\end{footnotesize}
Note the above formulation assumes each graph shall contain and only contain the common inlier nodes, which does not hold in the presence of unmatchable outliers. Thus it calls for more flexible solutions.

In this paper, we adopt a flexible approximate affinity score boosting procedure, which is gradually regularized by the overall matching consistency. The key idea is that the affinity score is indicative to semantic matching accuracy in the early stage of iterative boosting, while it becomes less informative as the accuracy lifting saturates. Then, in the presence of a few outliers, consistency becomes a useful regularizer to avoid the degenerating case between the two local graphs due to the fact that affinity score is often biased to semantic accuracy when the two graphs are corrupted by arbitrary noises, in addition with the inherent difficulty in modeling the affinity matrix using a compact parametric model. In the presence of more outliers, we propose a node-wise consistency/affinity-driven mechanism to elicit the affinity score and consistency relevant to common inlier nodes. Extensive empirical results illustrate the efficacy of our methods which is further specified in our conclusive remarks in Sec.\ref{sec:conclusion}.
\section{Proposed algorithms}\label{sec:main_methods}
Given the initial pairwise matching configuration $\mathbb{X}^{(0)}$ generated by a pairwise matching solver \emph{e.g.} \cite{ChoECCV10,ZhouCVPR12}, one can further derive a new pairwise matching $\textbf{X}^{(1)}_{ij}$ by the composition of a few `good' matchings \emph{i.e.} $\textbf{X}^{(1)}_{ij}=\textbf{X}^{(0)}_{ik_1}\textbf{X}^{(0)}_{k_1k_2}\ldots\textbf{X}^{(0)}_{k_{s-1}k_s}\textbf{X}^{(0)}_{k_sj}$ to replace the original $\textbf{X}^{(0)}_{ij}$ such that the matching accuracy can be improved or ideally maximized, by means of interpolating the sequence $\mathcal{G}_i, \mathcal{G}_{k_1}, \cdots, \mathcal{G}_{k_{s-1}}, \mathcal{G}_{k_s},\mathcal{G}_j$. The problem of this dynamic-replacing framework is how to set up the appropriate compositional replacements. In case the original matching is best, the replacement set shall be empty.

This paper devises a mechanism for iteratively finding the appropriate new composition to replace the old one without knowing their true accuracy. We will start with a baseline ISB (Alg.\ref{alg:ISB}), and two variants ISB$^{\text{2nd}}$, ISB$^{\text{cst}}$ for comparison. Then a graduated consistency-regularized method ISB-GC (Alg.\ref{alg:ISB-GC}) is highlighted with two efficient variants ISB-GC-U/P (Alg.\ref{alg:ISB-EGC}). Before diving into the details, we first introduce several notations and definitions which help the exposition of our methods.
\subsection{Notations and definitions}
Throughout the paper, $\mathbb{R}$ denotes for the real number domain. Bold capital letters denote for a matrix $\textbf{X}$, bold lower-case letters for a column vector $\textbf{x}$, and hollow bold letters for a set $\mathbb{X}$. All non-bold letters represent scalars. $\textbf{X}^T$ is the transpose of $\textbf{X}$; vec($\textbf{X}$) is the column-wise vectorized matrix $\textbf{X}$. tr($\textbf{X}$) is the trace of matrix $\textbf{X}$. $\textbf{I}_n \in \mathbb{R}^{n\times n}$ is an identity matrix and the subscript $n$ will be omitted when it can be inferred from context. $\textbf{1}_{m\times n},\textbf{0}_{m\times n}\in \mathbb{R}^{m\times n}$ is the matrix whose elements being all ones or zeros, respectively. $\textbf{1}_{n}$ is the abbreviation for $\textbf{1}_{n\times 1}$. $|\mathbb{X}|$ is the cardinality of the set $\mathbb{X}$, $\|\textbf{X}\|_F=\text{tr}(\textbf{X}^T\textbf{X})$ for the Frobenious norm and $\|\cdot\|_p$ is the $p$-norm.

This paper intensively uses $\mathbb{X}=\{\textbf{X}_{ij}\}_{i=1,j=i+1}^{N-1,N}$ to denote the set of pairwise matching matrices over $N$ graphs $\{\mathcal{G}_k\}_{k=1}^N$, which is also termed as \emph{matching configuration}, and $\mathbb{K}=\{\textbf{K}_{ij}\}_{i,j=1}^N$ is used for the affinity matrix set. In particular, we use $\mathbb{X}_c$ to denote a solution set which is fully consistent (referred by the subscript `c') as there is no contradictory pairwise matchings for $\{\textbf{X}_{ij}\neq\textbf{X}_{ik}\textbf{X}_{kj}\}_{i,j,k=1}^N$. If not otherwise stated explicitly, we adopt the convention that uses $N$ for the number of considered graphs, $n$ and $m$ for the number of nodes, and the number of edges in a graph, respectively.

So far we have not presented a formal definition regarding consistency, which has been widely used in \cite{YanICCV13,YanTIP15,ChenICML14,RibaltaCVIU11,PachauriNIPS13} and mentioned earlier in this paper. Now we give the formal definitions regarding matching consistency induced by the initial pairwise matching $\mathbb{X}$ from any pairwise matching solver. In addition, we also define two variants of the super graph and a concept related to matching composition. This paper does not claim the credit for the novelty of these definitions as similar definitions or concepts may appear in literature.
\begin{definition}\label{def:single_consist}
Given $N$ graphs $\{\mathcal{G}_k\}_{k=1}^N$ and the pairwise matching configuration $\mathbb{X}=\{\textbf{X}_{ij}\}_{i=1,j=i+1}^{N-1,N}$, the unary consistency of graph $\mathcal{G}_k$ is defined as $C_u(k,\mathbb{X})=1-\frac{\sum_{i=1}^{N-1}\sum_{j=i+1}^{N}\|\textbf{X}_{ij}-\textbf{X}_{ik}\textbf{X}_{kj}\|_F/2}{nN(N-1)/2}\in(0,1]$.
\end{definition}
\begin{definition}\label{def:pairwise_consist}
Given graphs $\{\mathcal{G}_k\}_{k=1}^N$ and matching configuration $\mathbb{X}$, for any pair $\mathcal{G}_i$ and $\mathcal{G}_j$, the pairwise consistency is defined as $C_{p}(\textbf{X}_{ij},\mathbb{X})=1-\frac{\sum_{k=1}^{N}\|\textbf{X}_{ij}-\textbf{X}_{ik}\textbf{X}_{kj}\|_F/2}{nN}\in(0,1]$.
\end{definition}
\begin{definition}\label{def:full_consist}
Given $N$ graphs $\{\mathcal{G}_k\}_{k=1}^N$ and $\mathbb{X}$, we call $\mathbb{X}$ is fully consistent \mbox{if and only if} $\frac{\sum_{i=1,j=i+1}^{N-1,N}C_{p}(\textbf{X}_{ij},\mathbb{X})}{N(N-1)/2}=1$ and (or) $\frac{\sum_{k=1}^{N}C_u(k,\mathbb{X})}{N}=1$. The following equation always holds: $\frac{\sum_{k=1}^{N}C_u(k,\mathbb{X})}{N}=\frac{\sum_{i=1,j=i+1}^{N-1,N}C_{p}(\textbf{X}_{ij},\mathbb{X})}{N(N-1)/2}$. We further define the overall consistency of $\mathbb{X}$ as $C(\mathbb{X})=\frac{\sum_{k=1}^{N}C_u(k,\mathbb{X})}{N}\in(0,1]$.
\end{definition}
\begin{definition}\label{def:node_consist}
Given $\{\mathcal{G}_k\}_{k=1}^N$ and $\mathbb{X}$, for node $\{\mathcal{N}_{u^k}\}_{u^k=1}^n$ in graph $\mathcal{G}_k$, its consistency w.r.t. $\mathbb{X}$ is defined by $C_\text{n}(u^k,\mathbb{X})=1-\frac{\sum_{i=1}^{N-1}\sum_{j=i+1}^{N}\|\textbf{Y}(u^k,:)\|_F/2}{N(N-1)/2}\in(0,1]$ where $\textbf{Y}=\textbf{X}_{kj}-\textbf{X}_{ki}\textbf{X}_{ij}$ and $\textbf{Y}(u^k,:)$ is the $u^k$th row of matrix $\textbf{Y}$.
\end{definition}
\begin{definition}\label{def:node_affinity}
Given $\{\mathcal{G}_k\}_{k=1}^N$, $\mathbb{X}$, $\mathbb{K}$, for node $\{\mathcal{N}_{u^k}\}_{u^k=1}^n$ in $\mathcal{G}_k$, its affinity w.r.t. $\mathbb{X}$ and $\mathbb{K}$ is defined by $S_n(u^k,\mathbb{X},\mathbb{K})=\sum_{i=1,\neq k}^N\text{vec}(\textbf{X}_{ki}^{u^k})^T\textbf{K}_{ki}\text{vec}(\textbf{X}_{ki})$, where $\textbf{X}^{u^k}$ denotes the matrix $\textbf{X}$ with zeros except for the $u^k$-th rows as is.
\end{definition}
\begin{definition}\label{def:super_affinity_graph}
Given $\{\mathcal{G}_k\}_{k=1}^{N}$ and $\mathbb{X}$, the affinity-wise super graph $\mathcal{G}^a_{sup}$ is defined as an undirected weighted graph s.t. each node $k$ denotes $\mathcal{G}_k$, and its edge weight $e_{ij}$ is the affinity score $J_{ij}(\textbf{X}_{ij})=\text{vec}(\textbf{X}_{ij})^T\textbf{K}_{ij}\text{vec}(\textbf{X}_{ij})$ induced by $\mathbb{X}$\footnote{In general $J_{ij}(\textbf{X}_{ij})\neq J_{ji}(\textbf{X}_{ji})$ and \cite{RibaltaSSPR10,RibaltaIJPRAI13} also add $J_{ij}(\textbf{X}_{ij})$ into the objective. While this setting convention is not related to the essence of our methods, thus for efficiency we use one-way affinity score to set the affinity-wise super graph and the objective in Eq.\ref{eq:overall_objective}.}.
\end{definition}
\begin{definition}\label{def:super_consistency_graph}
Given $\{\mathcal{G}_k\}_{k=1}^{N}$ and $\mathbb{X}$, the consistency-wise super graph $\mathcal{G}^c_{sup}$ is defined as an undirected weighted graph s.t. each node $k$ denotes $\mathcal{G}_k$, and the weight of edge $e_{ij}$ is the pairwise consistency  $C_{p}(\textbf{X}_{ij},\mathbb{X})=1-\frac{\sum_{k=1}^{N}\|\textbf{X}_{ij}-\textbf{X}_{ik}\textbf{X}_{kj}\|_F/2}{nN}$.
\end{definition}
\begin{definition}\label{def:score_path}
The path ${Z}_{ij}(k_1,k_2,\ldots,k_s)$ from $\mathcal{G}_i$ to $\mathcal{G}_j$ is defined as the chain $\mathcal{G}_i$$\rightarrow$$\mathcal{G}_{k_1}$$\rightarrow$$ \cdots$$\rightarrow $$\mathcal{G}_{k_s}$$\rightarrow$$\mathcal{G}_{j}$ which induces: $\textbf{Y}_{ij}(k_1,k_2,\ldots,k_s) \triangleq \textbf{X}_{ik_1}\textbf{X}_{k_1k_2}\ldots\textbf{X}_{k_sj}$.
Its order $s_{ij}$ is the number of the intermediate graphs between $\mathcal{G}_i$, $\mathcal{G}_j$. The path score is $J_{ij}(k_1,k_2,\ldots,k_s)=$\text{vec}$(\textbf{Y}_{ij})^T\textbf{K}_{ij}$\text{vec}$(\textbf{Y}_{ij})$.
\end{definition}
We present several comments to the above definitions. The unary consistency $C_u(k,\mathbb{X})$ by Definition (\ref{def:single_consist}) is a function \emph{w.r.t.} the graph index $k$ and $\mathbb{X}$. In contrast, the pairwise consistency $C_p(\textbf{X}_{ij},\mathbb{X})$ by Definition (\ref{def:pairwise_consist}) is generalized to the one \emph{w.r.t.} two variables of $\textbf{X}_{ij}$ and $\mathbb{X}$ and the former one $\textbf{X}_{ij}$ does not necessarily belong to $\mathbb{X}$. The pairwise consistency metric is symmetric such that $C_{p}(\textbf{X}_{ij},\mathbb{X}) = C_{p}(\textbf{X}_{ji},\mathbb{X})$, which derives Definition (\ref{def:full_consist}). Definition (\ref{def:node_consist}) and (\ref{def:node_affinity}) break down to node level consistency/affinity similar to unary ones, and they will be used in the proposed inlier eliciting mechanism.

The super graph, either for affinity-wise $\mathcal{G}^a_{sup}$ by Definition (\ref{def:super_affinity_graph}), or consistency-wise $\mathcal{G}^c_{sup}$ by Definition (\ref{def:super_consistency_graph}), is a connected (not necessarily fully connected) graph as a portion of pairs are matched by a certain means, a \emph{maximum spanning tree} (MST) \cite{GavrilJA87} can be found with no less total weight of every other spanning tree.

Note Definition (\ref{def:single_consist}) and Definition (\ref{def:pairwise_consist}) appeared in \cite{YanTIP15} in which they are used to identify the optimal alternating optimization sequence of variables. This paper utilize them differently to infuse the consistency along with the boosting procedure. The overall consistency for $\mathbb{X}$ as by Definition (\ref{def:full_consist}), further tells the relation between unary consistency and pairwise consistency. A similar overall `inconsistency' metric appears in \cite{RibaltaSSPR10,RibaltaIJPRAI13}.

The affinity-wise super graph $\mathcal{G}^a_{sup}$ by Definition (\ref{def:super_affinity_graph}) is similar to the one in \cite{CharpiatICCVW09}, of which the author proposes to build a graph where each shape is a vertex, and edges between shapes are weighted by the cost of the best matching. We put it in the scenario for the problem of weighted multi-graph matching. Furthermore, we give a similar definition for a consistency variant $\mathcal{G}^a_{sup}$ in terms of consistency-wise edge weights by Definition (\ref{def:super_consistency_graph}).

Finally, Definition (\ref{def:score_path}) is related to the approximate path selection idea as used in our algorithms. Obviously, when $\mathbb{X}$ is fully consistent, any two paths between two given graphs would induce the equal solution.

Based on the above definitions and discussion, we present our main approaches in the rest of this section.
\subsection{Iterative affinity score/consistency boosting}\label{subsec:affinity_boosting}
Most GM methods aim to maximize an objective function regarding with affinity score \cite{YanICCV13,YanTIP15,ChoECCV10,ZhouCVPR12,LeordeanuNIPS09}, we also follow this methodology in this subsection. Our basic rationale is that the node matching regarding with the highest affinity score between two graphs can be found along a higher-order path, as related to Definition (\ref{def:score_path}), instead of the direct (zero-order) pairwise matching. We formalize this idea as follows:

Without loss of generality, assume the affinity-wise super graph $\mathcal{G}^a_{sup}$ is fully connected which is induced by the configuration $\mathbb{X}$. Given $\mathcal{G}_i$, $\mathcal{G}_j$, all loop-free matching pathes on the super graph can form a set of loop-free paths $\mathbb{Z}_{ij}=\{Z_{ij}(k_1,\ldots,k_s)\}_{s=1}^{N-2}$, whose cardinality $|\mathbb{Z}_{ij}| = \sum_{s=1}^{N-2}s!$. Thus the overhead for finding the highest score solution is intractable since $\sum_{s=1}^{N-2}s!$ times of compositions need be computed. Given the matching configuration $\mathbb{X}^{(t-1)}$, one alternative is approximating the optimal $Z_{ij}^*$ by a series of iterations involving the first-order paths. And the problem of finding the optimal third-party graph $\mathcal{G}_k$ in iteration $t$ becomes:
\begin{small}
\begin{equation}\label{eq:score_max_compact}
k^*=\arg\max_{k=1}^NJ(\textbf{X}^{(t-1)}_{ik}\textbf{X}^{(t-1)}_{kj})
\end{equation}
\end{small}
This can be regarded as approximately concatenating the iteratively generated piecewise pathes, into a higher-order one that consists of multiple intermediate graphs.

The above idea is concretized into an iterative algorithm termed as \textbf{Iterative Score Boosting (ISB)} as described in the chart of Alg.\ref{alg:ISB}. In iteration $t$, every $\textbf{X}^{(t)}_{ij}$ is updated by seeking the path with order $s=1$ via maximizing the affinity score according to Eq.\ref{eq:score_max_compact}. The efficacy of such a score-boosting strategy can be justified by an intuitive analysis: even matching graph $\mathcal{G}_i$ and $\mathcal{G}_j$ is inherently ambiguous when both are largely corrupted, it is still possible to improve matching accuracy by an intermediate graph $\mathcal{G}_k$ that can find the quality pairwise matches between $\{\mathcal{G}_i,\mathcal{G}_k\}$ and $\{\mathcal{G}_k,\mathcal{G}_j\}$ respectively. The following proposition depicts the convergence of Alg.\ref{alg:ISB}.
\begin{proposition}\label{prop:ISB_convergence}
Alg.\ref{alg:ISB} (ISB) is ensured to converge to a stationary configuration $\mathbb{X}^*$ after a finite number of iterations.
\end{proposition}
\begin{proof}
For each $J(\textbf{X}_{ij})$, it forms a non-descending sequence over iterations which is bounded by the upper bound $\widetilde{J}_{ij}=\max\{\text{vec}(\textbf{Y})^T\textbf{K}_{ij}\text{vec}(\textbf{Y}), \textbf{Y}\in\mathbb{P}_n\}$ in the discrete permutation space $\mathbb{P}_n$: $J(\textbf{X}^{(0)}_{ij})\leq J(\textbf{X}^{(1)}_{ij})\leq\ldots\leq J(\textbf{X}^{(t)}_{ij})\leq\ldots\leq\widetilde{J}_{ij}$. Thus $\mathbb{X}^{(t)}$ will converge.
\end{proof}

In fact, Alg.\ref{alg:ISB} (ISB) cannot guarantee the convergent $\mathbb{X}^*$ satisfy full consistency \emph{i.e.} $C(\mathbb{X}^*)=1$ by Definition (\ref{def:full_consist}). Thus, after the iteration procedure, a post step for enforcing full consistency is \emph{optionally} performed. In our heuristical implementation, when the resultant $C(\mathbb{X}^*)$ is less than a given threshold $\gamma$, which suggests there are many contradictory matchings $\textbf{X}_{ij}\neq\textbf{X}_{ik}\textbf{X}_{kj}$, thus the consistency is too low to be informative. Then the affinity-wise super graph $\mathcal{G}^a_{sup}$ by Definition (\ref{def:super_affinity_graph}) is built to generate the final fully consistent $\mathbb{X}^*_c$ in the hope that affinity score might be more indicative. Otherwise, we utilize consistency to smooth the final solution in two cases: i) when $N>n$ \emph{i.e.} the number of graphs is larger than the number of nodes, then the method in \cite{PachauriNIPS13} is adopted to obtain $\mathbb{X}^*_c$; ii) otherwise, the consistency-wise super graph $\mathcal{G}_c$ by Definition (\ref{def:super_consistency_graph}) is built to obtain $\mathbb{X}^*_c$. We apply this policy in all our algorithms for post-processing to achieve full-consistency. Yet how to realize full consistency in post-step is not the focus of this paper which has also been addressed by other work such as \cite{PachauriNIPS13,ChenICML14}. Moreover, in the presence of a considerable number of outliers, enforcing the full consistency will cause performance degeneration due to the unmatchable outliers as will be shown in our experiments.

In order to obtain a comprehensive view of the affinity boosting model, we consider two additional variants of ISB (Alg.\ref{alg:ISB}). One uses the pairwise consistency instead of affinity to update solutions. The other replaces the first-order path selection with a second-order one. In fact, there can be other random search strategies aside from the fist-order one described here, and these random search variants can be derived by our framework.

The first variant involves replacing the updating step $\textbf{X}^{(t)}_{ij}=\textbf{X}^{(t-1)}_{ik}\textbf{X}^{(t-1)}_{kj}$ using Eq.\ref{eq:score_max_compact} with the pairwise consistency in Definition (\ref{def:pairwise_consist}), as formulated as follows:
\begin{small}
\begin{equation}\label{eq:consistency_max_compact}
k^*=\arg\max_{k=1}^NC_{p}(\textbf{X}^{(t-1)}_{ik}\textbf{X}^{(t-1)}_{kj},\mathbb{X}^{(t-1)})
\end{equation}
\end{small}

Unlike ISB, this consistency-driven variant ISB$^{\text{cst}}$ (note superscript added) cannot ensure to converge to a stationary configuration. In particular, the consistency boosting mechanism at each round, can only guarantee $C_p(\textbf{X}^{(t)}_{ij},\mathbb{X}^{(t-1)})\geq C_p(\textbf{X}^{(t-1)}_{ij},\mathbb{X}^{(t-1)})$ for any pair of $i$, $j$. Nevertheless, it cannot ensure $C_p(\textbf{X}^{(t)}_{ij},\mathbb{X}^{(t)})\geq C_p(\textbf{X}^{(t-1)}_{ij},\mathbb{X}^{(t-1)})$. Thus the non-decreasing property $C_p(\textbf{X}^{(0)}_{ij},\mathbb{X}^{(0)})\leq C_p(\textbf{X}^{(1)}_{ij},\mathbb{X}^{(1)})\leq\ldots$ does not hold, although they are bounded by $C_{p}(\textbf{X}^{(t)}_{ij},\mathbb{X}^{(t)})\leq1$. For similar reasons, one cannot ensure $C(\mathbb{X}^{(t)})\geq C(\mathbb{X}^{(t-1)})$ where $C(\cdot)$ is defined in Definition (\ref{def:full_consist}). As the solution space is discrete and finite, ISB$^{\text{cst}}$ either converges to a stationary point or forms a looping solution path. The former case is much more often observed in our tests.

The 2nd-order variant ISB$^{\text{2nd}}$ is derived by replacing $\textbf{X}^{(t)}_{ij}=\textbf{X}^{(t-1)}_{ik}\textbf{X}^{(t-1)}_{kj}$ via $\textbf{X}^{(t)}_{ij}=\textbf{X}^{(t-1)}_{iv}\textbf{X}^{(t-1)}_{vu}\textbf{X}^{(t-1)}_{uj}$:
\begin{small}
\begin{equation}\label{eq:score_max_2nd}
u^*,v^*=\arg\max_{u,v=1}^N\textbf{y}^T\textbf{K}_{ij}\textbf{y},\quad \textbf{y}=\text{vec}(\textbf{X}^{(t-1)}_{iv}\textbf{X}^{(t-1)}_{vu}\textbf{X}^{(t-1)}_{uj})
\end{equation}
\end{small}
The 2nd-order method has more exploration capability than the first-order one, at the expense of growing searching space from $O(N)$ to $O(N^2)$ in terms of traversing the rest graphs (or graph pairs for the 2nd-order case). As the same with ISB, the 2nd-order method can also guarantee convergence due to its score boosting property. These three methods are evaluated in Fig.\ref{fig:case_study} together with Alg.\ref{alg:ISB-GC} (ISB-GC) as will be introduced later. Moreover, in the ideal case when the ground truth is know, then we modify the method ISB by boosting the true pairwise accuracy instead of the affinity score for iterative updating. It is termed as ISB$^{\text{acc}}$ and shown in all the plots which can be regarded as the accuracy upper-bound given the initial $\mathbb{X}^{(0)}$.

One can see that ISB outperforms ISB$^{\text{cst}}$ notably. This is because ISB$^{\text{cst}}$ is based on the assumption that correct node-to-node matchings are dominant and the meaningful correspondences can be realized along many different paths of maps. This assumption would break given an unsatisfactory initial $\mathbb{X}^{(0)}$, which in turn calls for an affinity boosting step to satisfy the assumption. For ISB$^{\text{2nd}}$, it outperforms the first-order method, while its overhead becomes significantly larger when more graphs are considered. We leave the two methods Alg.\ref{alg:ISB-GC} (ISB-GC) and its variant ISB-GC$^{\text{inv}}$ which also both appear in Fig.\ref{fig:case_study} to the next sub-section.
\begin{algorithm}[tb]
\small
  \caption{\small Iterative Score boosting (ISB) \normalsize}
  \label{alg:ISB}
  \begin{algorithmic}[1]
    \REQUIRE $\{\textbf{K}_{ij}\}_{i=1,j=i+1}^{N-1,N}$, $T_{max}$, $\delta$, $\gamma\in(0,1)$;
      \STATE Perform pairwise matching to obtain initial $\mathbb{X}^{(0)}$;
      \STATE Calculate $J^{(0)}=\sum_{i=1,j=i+1}^{N-1,N}{\text{vec}(\textbf{X}_{ij}^{(0)})}^T\textbf{K}_{ij}\text{vec}(\textbf{X}^{(0)}_{ij})$;
    \FOR {$t=1:T_{max}$}
        \FORALL {$i=1,2,\dots,N-1;j=i+1,\ldots,N$}
        \STATE update $\textbf{X}^{(t)}_{ij}=\textbf{X}^{(t-1)}_{ik}\textbf{X}^{(t-1)}_{kj}$ by solving Eq.\ref{eq:score_max_compact};
        \ENDFOR
        \STATE \textbf{if} $\sum_{i=1,j=i+1}^{N-1,N}\|\textbf{X}^{(t-1)}_{ij}-\textbf{X}^{(t)}_{ij}\|_2<\delta$, \textbf{break};
    \ENDFOR
    \STATE \textbf{if} $C(\mathbb{X}^{(t)})=1$, \textbf{return} $\mathbb{X}^*_c=\mathbb{X}^{(t)}$;
    \IF {$C(\mathbb{X}^{(t)})<\gamma$}\label{alg:isl:enforce_consistency_begin}
    \STATE Build the super graph $\mathcal{G}^a_{sup}$ by pairwise affinity score $J(\textbf{X}^{(t)}_{ij})$ and find a maximum span tree to generate $\mathbb{X}^*_c$;
    \ELSE
        \IF{$n\geq N$}
        \STATE Build the super graph $\mathcal{G}^c_{sup}$ by $C_p(\textbf{X}^{(t)}_{ij},\mathbb{X}^{(t)})$ and find a maximum span tree to generate $\mathbb{X}^*_c$;
        \ELSE
        \STATE Perform the smoothing method \cite{PachauriNIPS13} to obtain $\mathbb{X}^*_c$;
        \ENDIF
    \ENDIF\label{alg:isl:enforce_consistency_end}
  \end{algorithmic}
\end{algorithm}
\begin{algorithm}[tb]
\small
  \caption{\small Graduated Consistency-regularized Iterative Score Boosting (ISB-GC)\normalsize}
  \label{alg:ISB-GC}
  \begin{algorithmic}[1]
    \REQUIRE $\{\textbf{K}_{ij}\}_{i=1,j=i+1}^{N-1,N}$; $T_{max}$, $\delta$, $\lambda^{(T_0)}=\lambda^{(0)}$, $s$, $\gamma$;
      \STATE Perform pairwise matching to obtain initial $\mathbb{X}^{(0)}$;
      \STATE Calculate $J^{(0)}=\sum_{i=1,j=i+1}^{N-1,N}{\text{vec}(\textbf{X}_{ij}^{(0)})}^T\textbf{K}_{ij}\text{vec}(\textbf{X}^{(0)}_{ij})$;
    \FOR {$t=1:T_{0}$}\label{alg:gcisl:pure_affinity_begin}
        \FORALL {$i=1,2,\dots,N-1;j=i+1,\ldots,N$}
        \STATE update $\textbf{X}^{(t)}_{ij}=\textbf{X}^{(t-1)}_{ik}\textbf{X}^{(t-1)}_{kj}$ by solving Eq.\ref{eq:score_max_compact};
        \ENDFOR
    \ENDFOR\label{alg:gcisl:pure_affinity_end}
    \FOR {$t=T_{0}+1:T_{max}$}
        \FORALL {$i=1,2,\dots,N-1;j=i+1,\ldots,N$}
        \STATE update $\textbf{X}^{(t)}_{ij}=\textbf{X}^{(t-1)}_{ik}\textbf{X}^{(t-1)}_{kj}$ by solving Eq.\ref{eq:score_consistency_max};
        \ENDFOR
        \STATE \textbf{if} $\sum_{i=1,j=i+1}^{N-1,N}\|\textbf{X}^{(t-1)}_{ij}-\textbf{X}^{(t)}_{ij}\|_2<\delta$, \textbf{break};
        \STATE $\lambda^{(t)}=min(1,\beta\lambda^{(t-1)}$);
    \ENDFOR
    \STATE Perform the same post-processing as in Alg.\ref{alg:ISB} (L\ref{alg:isl:enforce_consistency_begin}-\ref{alg:isl:enforce_consistency_end}).
  \end{algorithmic}
\end{algorithm}
\subsection{Graduated consistency-regularized boosting}\label{subsec:consistency_boosting}
Our key rationale is viewing the consistency constraint as a \emph{regularizer} for affinity score maximization. Note that maximizing pairwise matching score among all pairs cannot ensure the consistency constraint. Moreover, due to outliers, sparse edge sampling for efficiency, and local deformation as well as the difficulty in parameterizing the affinity matrix as already discussed in \cite{YanTIP15} and the references therein, there can be a case that for some pairs of graphs, the semantic ground truth matching configuration may not correspond to the highest affinity score. Thus purely maximizing the overall matching score is biased to accuracy, though maximizing overall consistency alone is even more biased as Alg.\ref{alg:ISB}$^{\text{cst}}$ which has been studied in the previous sub-section. Informally speaking, this is analogous to loss function modeling in machine learning, where one not only considers empirical loss on the training dataset, but also introduces regularizer to account for over-fitting.

One shall note as a baseline method, ISB separates affinity score maximization and consistency smoothing into two independent steps. It is yet appealing to tackle the two aspects jointly. We make the following statements, for devising an effective algorithm that gradually introduces consistency during the score-boosting procedure. It is motivated by two observations: i) For the initial assignment matrix $\mathbb{X}^{(0)}$ obtained by the pairwise graph matching solver, its scores are more informative for the true accuracy; ii) After several rounds of iterations for score-boosting, affinity score becomes less discriminative and consistency becomes more indicative.

In this spirit, we infuse the matching consistency by a weighted term, whose weight $\lambda$ is gradually increased. As a result, the evaluation function in each iteration is changed from Eq.\ref{eq:score_max_compact} to a weighted one\footnote{Note the affinity score is not directly comparable with either the unary or pairwise consistency as the latter fall in $[0,1]$ while the former is arbitrary depending on how the affinity matrix is set. In our implementation, the affinity score is further normalized by $J(\textbf{X}_{ij})=J(\textbf{X}_{ij})/\max_{ij}\{J(\textbf{X}^{(0)}_{ij})\}$ where the denominator is a constant. This convention is also used when $J^{\psi}$ is introduced in Sec.\ref{subsec:outlier_adapt}.} between Eq.\ref{eq:score_max_compact} and Eq.\ref{eq:consistency_max_compact} by setting $\textbf{Y}_{ikj}^{(t-1)}=\textbf{X}^{(t-1)}_{ik}\textbf{X}^{(t-1)}_{kj}$:
\begin{small}
\begin{equation}\label{eq:score_consistency_max}
k^*=\arg\max_{k=1}^N(1-\lambda)J(\textbf{Y}_{ikj}^{(t-1)})+\lambda C_{p}(\textbf{Y}_{ikj}^{(t-1)},\mathbb{X}^{(t-1)})
\end{equation}
\end{small}

The similar post-processing that enforces full consistency in Alg.\ref{alg:ISB} (ISB) can also be conducted here which will stop the score growing immediately. This method is detailed in the chart of Alg.\ref{alg:ISB-GC}: \textbf{Graduated Consistency-regularized Iterative Score Boosting (ISB-GC)}.

Similar to ISB$^{\text{cst}}$, ISB-GC in general cannot guarantee to converge to a stationary configuration, since the non-decreasing property of the sequence $\{C_p(\textbf{X}^{(t)}_{ij},\mathbb{X}^{(t)})\}_{t=0}^\infty$ cannot always hold.
Our empirical tests show this method can often converge to a stationary $\mathbb{X}^*$ after 10 iterations or so. In our experiments, we stop it when the number of iterations arrives a certain threshold $T_{max}$.

To make our study more comprehensive, like ISB$^{\text{cst}}$ to ISB, we further devise a counterpart to ISB-GC by swapping the role of consistency and affinity in this algorithm: the iterative updating is driven by choosing the anchor graph that lifts the pairwise consistency $C_p(\textbf{X}^{(t)}_{ij},\mathbb{X}^{(t-1)})$ as denoted by Eq.\ref{eq:consistency_max_compact} for the method ISB$^{\text{cst}}$. Meanwhile, the weight of the affinity score is gradually increased as the same as the consistency used in ISB-GC. Accordingly, the evaluation function becomes:
\begin{small}
\begin{equation}\label{eq:score_consistency_max}
k^*=\arg\max_{k=1}^N\lambda J(\textbf{Y}_{ikj}^{(t-1)})+(1-\lambda) C_{p}(\textbf{Y}_{ikj}^{(t-1)},\mathbb{X}^{(t-1)})
\end{equation}
\end{small}
We call this algorithm ISB-GC$^{\text{inv}}$ where the superscript denotes for `inverse'. The parameter settings $\lambda^0$ and $\beta$ are identical to ISB-GC. We omit the algorithm chart for ISB-GC$^{\text{inv}}$ since it is akin to ISB-GC. It is outperformed by ISB-GC as shown in Fig.\ref{fig:case_study} which validates our idea.
\begin{figure*}[tb]
\centering
\setlength{\abovecaptionskip}{0pt}
\setlength{\belowcaptionskip}{-10pt}
\subfigure{\label{fig:legend}
\includegraphics[width=0.96\textwidth]{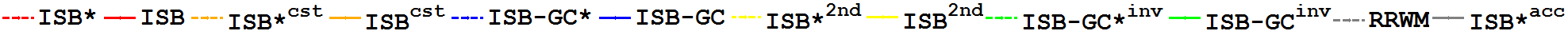}}\\\vspace{-10pt}
\setcounter{subfigure}{0}
\subfigure[Deform by $\varepsilon$]{\label{fig:random_case_vary_deform_rrwm_acc}
\includegraphics[width=0.19\textwidth]{{{performance_random_deform_RRWM_RRWM_respect_0.08_0.18_case_acc}}}}\hspace{-7pt}
\subfigure[Outlier by $n_o$]{\label{fig:random_case_vary_outlier_rrwm_acc}
\includegraphics[width=0.19\textwidth]{{{performance_random_outlier_RRWM_RRWM_respect_1_6_case_acc}}}}\hspace{-7pt}
\subfigure[Density by $\rho$]{\label{fig:random_case_vary_density_rrwm_acc}
\includegraphics[width=0.19\textwidth]{{{performance_random_density_RRWM_RRWM_respect_0.45_0.60_case_acc}}}}\hspace{-7pt}
\subfigure[Coverage by $c$]{\label{fig:random_case_vary_coverage_rrwm_acc}
\includegraphics[width=0.19\textwidth]{{{performance_random_complete_RRWM_RRWM_respect_0.05_0.30_case_acc}}}}\hspace{-7pt}
\subfigure[Outlier by $n_o$ (time)]{\label{fig:random_case_vary_outlier_rrwm_tim}
\includegraphics[width=0.19\textwidth]{{{performance_random_outlier_RRWM_RRWM_respect_1_6_case_tim}}}}\\\vspace{-5pt}
\subfigure[Deform by $N$]{\label{fig:random_case_vary_graph_cnt_deform_rrwm_acc}
\includegraphics[width=0.19\textwidth]{{{performance_random_deform_RRWM_RRWM_respect_case_all_acc}}}}\hspace{-7pt}
\subfigure[Outlier by $N$]{\label{fig:random_case_vary_graph_cnt_outlier_rrwm_acc}
\includegraphics[width=0.19\textwidth]{{{performance_random_outlier_RRWM_RRWM_respect_case_all_acc}}}}\hspace{-7pt}
\subfigure[Density by $N$]{\label{fig:random_case_vary_graph_cnt_density_rrwm_acc}
\includegraphics[width=0.19\textwidth]{{{performance_random_density_RRWM_RRWM_respect_case_all_acc}}}}\hspace{-7pt}
\subfigure[Coverage by $N$]{\label{fig:random_case_vary_graph_cnt_coverage_rrwm_acc}
\includegraphics[width=0.19\textwidth]{{{performance_random_complete_RRWM_RRWM_respect_case_all_acc}}}}\hspace{-7pt}
\subfigure[Deform by $N$ (time)]{\label{fig:random_case_vary_graph_cnt_deform_rrwm_tim}
\includegraphics[width=0.19\textwidth]{{{performance_random_deform_RRWM_RRWM_respect_case_all_tim}}}}\\
\caption{Comparison of ISB, ISB$^{\text{cst}}$, ISB$^{\text{2nd}}$, ISB-GC and ISB-GC$^{\text{inv}}$ on the synthetic random graph data by varying the noise (top row) and by the number of graphs (bottom row). RRWM \cite{ChoECCV10} is used as the pairwise matcher to generate the initial configuration (`RRWM'). `ISB*$^{\text{acc}}$' denotes the ideal `upper-bound' by boosting true accuracy rather than affinity or consistency. `ISB*' denotes ISB dismissing the step of L\ref{alg:isl:enforce_consistency_begin}-\ref{alg:isl:enforce_consistency_end}, so for other methods. Refer to Table.\ref{tab:case_study_setting} for the settings.}
\label{fig:case_study}
\end{figure*}
\begin{figure*}[tb]
\centering
\setlength{\abovecaptionskip}{0pt}
\setlength{\belowcaptionskip}{-10pt}
\subfigure{\label{fig:legend}
\includegraphics[width=0.8\textwidth]{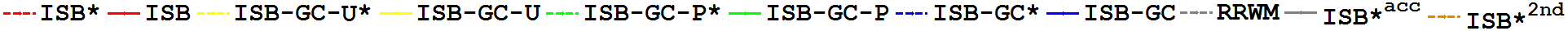}}\\\vspace{-10pt}
\setcounter{subfigure}{0}
\subfigure[Deform by $T_{max}$]{\label{fig:random_iter_deform_rrwm_acc}
\includegraphics[width=0.16\textwidth]{{{performance_random_deform_RRWM_RRWM_respect_iter_all_acc}}}}\hspace{-7pt}
\subfigure[Outlier by $T_{max}$]{\label{fig:random_iter_outlier_rrwm_acc}
\includegraphics[width=0.16\textwidth]{{{performance_random_outlier_RRWM_RRWM_respect_iter_all_acc}}}}\hspace{-7pt}
\subfigure[Density test by $T_{max}$]{\label{fig:random_iter_density_rrwm_acc}
\includegraphics[width=0.16\textwidth]{{{performance_random_density_RRWM_RRWM_respect_iter_all_acc}}}}\hspace{-7pt}
\subfigure[Coverage by $T_{max}$]{\label{fig:random_iter_coverage_rrwm_acc}
\includegraphics[width=0.16\textwidth]{{{performance_random_complete_RRWM_RRWM_respect_iter_all_acc}}}}\hspace{-7pt}
\subfigure[Outlier by $T_{max}$ (time)]{\label{fig:random_iter_outlier_rrwm_tim}
\includegraphics[width=0.16\textwidth]{{{performance_random_outlier_RRWM_RRWM_respect_iter_all_tim}}}}\hspace{-7pt}
\subfigure[By search sample rate]{\label{fig:random_sample_rate_rrwm_acc}
\includegraphics[width=0.16\textwidth]{{{performance_random_deform_RRWM_RRWM_respect_0.00_1.00_sample_acc}}}}\\\vspace{-5pt}
\caption{Performance on synthetic data by varying the iteration threshold $T_{max}$ for the proposed algorithms. Refer to Table.\ref{tab:iter_study_setting} for the corresponding settings.}
\label{fig:iter_study}
\end{figure*}
\begin{table}[t]
\center
\small
\caption{Settings for synthetic random graph test in Fig.\ref{fig:case_study} and Fig.\ref{fig:formal_syn_test}.}
\resizebox{0.47\textwidth}{!}{
\begin{tabular}{|rrr|}
  \hline
  x-axis&fixed parameters&results\\\hline
  $\varepsilon$=.08-.18&$N$=30,$n_{i}$=10,$n_{o}$=0,$\rho$=.9,$\sigma$=.05,$c$=1,$T_{max}$=6&Fig.\ref{fig:random_case_vary_deform_rrwm_acc}, Fig.\ref{fig:random_formal_vary_deform_rrwm_acc}\\
  $n_{o}$=1-6&$N$=30,$\varepsilon$=.05,$n_{i}$=8,$\rho$=1,$\sigma$=.05,$c$=1,$T_{max}$=6&Fig.\ref{fig:random_case_vary_outlier_rrwm_acc}, Fig.\ref{fig:random_formal_vary_outlier_rrwm_acc}\\
  $\rho$=.45-.6&$N$=30,$n_{i}$=10,$n_{o}$=0,$\varepsilon$=.05,$\sigma$=.05,$c$=1,$T_{max}$=6&Fig.\ref{fig:random_case_vary_density_rrwm_acc}, Fig.\ref{fig:random_formal_vary_density_rrwm_acc}\\
  c=.05-.3&$N$=30,$n_{i}$=10,$n_{o}$=0,$\varepsilon$=.05,$\rho=$1,$\sigma$=.05,$T_{max}$=6&Fig.\ref{fig:random_case_vary_coverage_rrwm_acc}, Fig.\ref{fig:random_formal_vary_coverage_rrwm_acc}\\\hline
  $N$=4-32&$n_{i}$=10,$n_{o}$=0,$\varepsilon$=.15,$\rho$=.9,$\sigma$=.05,$c$=1,$T_{max}$=6&Fig.\ref{fig:random_case_vary_graph_cnt_deform_rrwm_acc}, Fig.\ref{fig:random_formal_vary_graph_cnt_deform_rrwm_acc}\\
  $N$=4-32&$n_{i}$=6,$n_{o}$=4,$\varepsilon$=0,$\rho$=1,$\sigma$=.05,$c$=1,$T_{max}$=6&Fig.\ref{fig:random_case_vary_graph_cnt_outlier_rrwm_acc}, Fig.\ref{fig:random_formal_vary_graph_cnt_outlier_rrwm_acc}\\
  $N$=4-32&$n_{i}$=10,$n_{o}$=0,$\varepsilon$=0,$\rho$=.5,$\sigma$=.05,$c$=1,$T_{max}$=6&Fig.\ref{fig:random_case_vary_graph_cnt_density_rrwm_acc}, Fig.\ref{fig:random_formal_vary_graph_cnt_density_rrwm_acc}\\
  $N$=4-32&$n_{i}$=10,$n_{o}$=0,$\varepsilon$=.05,$\rho$=1,$\sigma$=.05,$c$=.1,$T_{max}$=6&Fig.\ref{fig:random_case_vary_graph_cnt_coverage_rrwm_acc}, Fig.\ref{fig:random_formal_vary_graph_cnt_coverage_rrwm_acc}\\\hline
\end{tabular}\label{tab:case_study_setting}}
\normalsize
\end{table}
\begin{table}[t]
\center
\small
\caption{Settings for synthetic random graph test in Fig.\ref{fig:iter_study}.}
\resizebox{0.47\textwidth}{!}{
\begin{tabular}{|rrrr|}
  \hline
  x-axis&test mode&fixed parameters&results\\\hline
  $T_{max}$=0-6&deform&$N$=20,$n_{i}$=10,$n_{o}$=0,$\varepsilon$=.15,$\rho$=.9,$c$=1,$\sigma$=.05&Fig.\ref{fig:random_iter_deform_rrwm_acc}\\
  $T_{max}$=0-6&outlier&$N$=20,$n_{i}$=6,$n_{o}$=4,$\varepsilon$=0,$\rho$=1,$c$=1,$\sigma$=.05&Fig.\ref{fig:random_iter_outlier_rrwm_acc}\\
  $T_{max}$=0-6&density&$N$=20,$n_{i}$=10,$n_{o}$=0,$\varepsilon$=0,$\rho$=.5,$c$=1,$\sigma$=.05&Fig.\ref{fig:random_iter_density_rrwm_acc}\\
  $T_{max}$=0-6&coverage&$N$=20,$n_{i}$=10,$n_{o}$=0,$\varepsilon$=.05,$\rho$=1,$c$=.1,$\sigma$=.05&Fig.\ref{fig:random_iter_coverage_rrwm_acc}\\
  sample rate=0-1&mixture&$N$=20,$n_{i}$=8,$n_{o}$=2,$\varepsilon$=.1,$\rho$=1,$c$=1,$\sigma$=.05&Fig.\ref{fig:random_sample_rate_rrwm_acc}\\
  \hline
\end{tabular}\label{tab:iter_study_setting}}
\normalsize
\end{table}
\subsection{Efficient consistency-regularized boosting}\label{subsec:efficient_consistency_boosting}
The step of computing $C_{p}(\textbf{X}^{(t-1)}_{ik}\textbf{X}^{(t-1)}_{kj},\mathbb{X}^{(t-1)})$ in ISB-GC need be repeated for each new $k$ because it depends on the dynamically generated candidate $\textbf{X}^{(t-1)}_{ik}\textbf{X}^{(t-1)}_{kj}$. As a result, its time complexity regarding the consistency term is $O(N^2n^3)$ for updating each $\textbf{X}^{(t)}_{ij}$ and the overall cost is $O(N^4n^3)$ in each iteration of ISB-GC. To reduce this overhead, we are motivated to use two more efficient consistency metrics to delegate $C_{p}(\textbf{X}^{(t-1)}_{ik}\textbf{X}^{(t-1)}_{kj},\mathbb{X}^{(t-1)})$.

The first proxy metric is the unary consistency as described in Definition (\ref{def:single_consist}). Note this metric is also what we adopt in our preliminary work \cite{YanECCV14} where it is referred as `Algorithm 2'. Its evaluation function is:
\begin{small}
\begin{equation}\label{eq:score_graph_consistency_max}
k^*=\arg\max_{k=1}^N(1-\lambda)J(\textbf{X}^{(t-1)}_{ik}\textbf{X}^{(t-1)}_{kj})+\lambda C_u(k,\mathbb{X}^{(t-1)})
\end{equation}
\end{small}
The merit of this metric is all $\{C_u(k,\mathbb{X}^{(t-1)})\}_{k=1}^{N}$ can be pre-computed in each iteration only for once.

Alternatively, we devise another consistency estimator by first pre-computing all $\{C_{p}(\textbf{X}_{ij}^{(t-1)},\mathbb{X}^{(t-1)})\}_{i=1,j=i+1}^{N-1,N}$ in the beginning of each iteration $t-1$, and then adopting Eq.\ref{eq:score_pair_consistency_estimator_max} in the hope of approximating $C_p(\textbf{X}^{(t-1)}_{ik}\textbf{X}^{(t-1)}_{kj},\mathbb{X}^{(t-1)})$ as reused in updating $\textbf{X}^{(t)}_{ij}$:
\begin{small}
\begin{align}\label{eq:score_pair_consistency_estimator_max}
k^*=\arg\max_{k=1}^N&(1-\lambda)J(\textbf{X}^{(t-1)}_{ik}\textbf{X}^{(t-1)}_{kj})\\\notag
&+\lambda\sqrt{C_{p}(\textbf{X}^{(t-1)}_{ik},\mathbb{X}^{(t-1)})C_{p}(\textbf{X}^{(t-1)}_{kj},\mathbb{X}^{(t-1)})}
\end{align}
\end{small}
\begin{algorithm}[tb]
\small
  \caption{\small Unary/Pairwise Graduated Consistency-regularized Iterative Score Boosting (ISB-GC-U/P) \normalsize}
  \label{alg:ISB-EGC}
  \begin{algorithmic}[1]
    \REQUIRE $\{\textbf{K}_{ij}\}_{i=1,j=i+1}^{N-1,N}$, $T_{max}$, $\delta$, $\gamma$, $\lambda^{(T_0)}=\lambda^{(0)}$, $s$, mode;
      \STATE Perform pairwise matching to obtain initial $\mathbb{X}^{(0)}$;
      \STATE Calculate $J^{(0)}=\sum_{i=1,j=i+1}^{N-1,N}{\text{vec}(\textbf{X}_{ij}^{(0)})}^T\textbf{K}_{ij}\text{vec}(\textbf{X}^{(0)}_{ij})$;
    \STATE Run L\ref{alg:gcisl:pure_affinity_begin}-\ref{alg:gcisl:pure_affinity_end} in Alg.\ref{alg:ISB-GC} for pure affinity score driven boosting;
    \FOR {$t=T_{0}+1:T_{max}$}
        \IF {mode=`unary'}
            \STATE Compute unary consistency $\{C_u(k,\mathbb{X}^{(t-1)})\}_{k=1}^{N}$;
        \ELSIF{mode=`pairwise'}
            \STATE Update pairwise metric $\{C_{p}(\textbf{X}_{ij}^{(t-1)},\mathbb{X}^{(t-1)})\}_{i=1,j=i+1}^{N-1,N}$;
        \ENDIF
        \FORALL {$i=1,2,\dots,N-1;j=i+1,\ldots,N$}
        \IF {mode=`unary'}
            \STATE update $\textbf{X}^{(t)}_{ij}=\textbf{X}^{(t-1)}_{ik}\textbf{X}^{(t-1)}_{kj}$ by solving Eq.\ref{eq:score_graph_consistency_max};
        \ELSIF{mode=`pairwise'}
            \STATE update $\textbf{X}^{(t)}_{ij}=\textbf{X}^{(t-1)}_{ik}\textbf{X}^{(t-1)}_{kj}$ by solving Eq.\ref{eq:score_pair_consistency_estimator_max};
        \ENDIF
        \ENDFOR
        \STATE \textbf{if} $\sum_{i=1,j=i+1}^{N-1,N}\|\textbf{X}^{(t-1)}_{ij}-\textbf{X}^{(t)}_{ij}\|_2<\delta$, \textbf{break};
        \STATE $\lambda^{(t)}=min(1,\beta\lambda^{(t-1)}$);
    \ENDFOR
    \STATE Perform the same post-processing as in Alg.\ref{alg:ISB} (L\ref{alg:isl:enforce_consistency_begin}-\ref{alg:isl:enforce_consistency_end}).
  \end{algorithmic}
\end{algorithm}
We term the above two efficient variants as \textbf{Unary/Pairwise Graduated Consistency-regularized Iterative Score Boosting} as depicted in the chart of Alg.\ref{alg:ISB-EGC}. In this paper, they are further abbreviated by ISB-GC-U and ISB-GC-P, where the last characters denote for `unary' and `pairwise' respectively. In particular, ISB-GC-U has a convergence property as stated in below.
\begin{proposition}\label{prop:approximate_graduated_iteration}
ISB-GC-U will converge to a stationary $\mathbb{X}^*$.
\end{proposition}
\begin{proof}
Given two graphs $\mathcal{G}_i$, $\mathcal{G}_j$ of $n$ nodes for each, first, define the set of score difference $\{\Delta J_{ij}\}$ as $\Delta J_{ij}=|\text{vec}(\textbf{X})^T\textbf{K}_{ij}\text{vec}(\textbf{X})-\text{vec}(\textbf{Y})^T\textbf{K}_{ij}\text{vec}(\textbf{Y})|$, $\forall \textbf{X}, \textbf{Y}$, between any two assignment matrices $\textbf{X}$, $\textbf{Y} \in \mathbb{R}^{n\times n}$ in the enumerable permutation space. Suppose the largest value of difference is $\Delta \widetilde{J}_{ij}$ which is constant given $\textbf{K}_{ij}$. We further define the largest difference denoted by $\Delta \widetilde{J}_{\mathcal{G}}=\max\{\Delta \widetilde{J}_{ij}\}_{i=1,j=i+1}^{N-1,N}$ for all pairs in the graph.

For a certain iteration in ISB-GC-U, suppose the most consistent graph by Definition (\ref{def:single_consist}) is $\mathcal{G}_a$, and the second one is $\mathcal{G}_b$. In iteration $t$ it will finally satisfy the condition: $\Delta C^{(t)}_u=C_u(a,\mathbb{X}^{(t)})-C_u(b,\mathbb{X}^{(t)})>\frac{(1-\lambda)}{\lambda}\Delta \widetilde{J}_{\mathcal{G}}$ as $\lambda^{(t)}\rightarrow1$ by $\lambda^{(t)}=min(\rho\lambda^{(t-1)},1)$\footnote{In case $C_u(a,\mathbb{X}^{(t)})=C_u(b,\mathbb{X}^{(t)})$, without loss of generality, one can choose any of them as the largest one, and choose the next $\mathcal{G}_c$ as the second largest if $C_u(a,\mathbb{X}^{(t)})>C_u(c,\mathbb{X}^{(t)}).$}. Then, the following inequality will hold for any $k\neq a$ and $\{\mathcal{G}_i\}_{i=1}^{N-1}, \{\mathcal{G}_j\}_{j=i+1}^{N}$:
\begin{small}
\begin{align}\label{eq:graph_wise_bound}
&(1-\lambda)J(\textbf{X}^{(t)}_{ia}\textbf{X}^{(t)}_{aj})+\lambda C_u(a,\mathbb{X}^{(t-1)})\\\notag
=&(1-\lambda)J(\textbf{X}^{(t)}_{ia}\textbf{X}^{(t)}_{aj})+\lambda C_u(b,\mathbb{X}^{(t-1)})+\lambda\Delta C_u^{(t)}\\\notag
=&(1-\lambda)\left(J(\textbf{X}^{(t)}_{ia}\textbf{X}^{(t)}_{aj})-J(\textbf{X}^{(t)}_{ik}\textbf{X}^{(t)}_{kj})\right)+(1-\lambda)J(\textbf{X}^{(t)}_{ik}\textbf{X}^{(t)}_{kj})\\\notag
&+\lambda C_u(b,\mathbb{X}^{(t-1)})+\lambda\Delta C_u^{(t)}\\\notag
\geq&\lambda\Delta C^{(t)}_u-\Delta \widetilde{J}_{\mathcal{G}}+(1-\lambda)J(\textbf{X}^{(t)}_{ik}\textbf{X}^{(t)}_{kj})+\lambda C_u(k,\mathbb{X}^{(t-1)})\\\notag
>&(1-\lambda)J(\textbf{X}^{(t)}_{ik}\textbf{X}^{(t)}_{kj})+\lambda C_u(k,\mathbb{X}^{(t-1)})
\end{align}
\end{small}
As a result, all $\{\textbf{X}^{(t+1)}_{ij}\}_{i=1,j=i+1}^{N-1,N}$ will be updated by $\textbf{X}^{(t+1)}_{ij}=\textbf{X}^{(t)}_{ia}\textbf{X}^{(t)}_{aj}$. In fact, in iteration $t$, $\mathbb{X}^{(t)}$ become fully consistent by Definition (\ref{def:full_consist}) thus cannot be lifted by either affinity or consistency metrics used in this paper.
\end{proof}
The iterative procedure in ISB-GC-U may exceed the maximum iterating threshold $T_{max}$ before reaching the above inequality holds when $\lambda$ is assumed close to $1$ enough. Thus a post-processing step is also needed in line with Alg.\ref{alg:ISB} (ISB) and Alg.\ref{alg:ISB-GC} (ISB-GC) if necessary.

For the `pairwise' version ISB-GC-P, this method in our experiments also often converges to a stationary $\mathbb{X}^{*}$ similar to ISB-GC, yet there is no theoretical guarantee. Fig.\ref{fig:iter_study} illustrates how the overall matching accuracy is lifted via the proposed algorithms, controlled by Table.\ref{tab:iter_study_setting}.
\subsection{Eliciting affinity and consistency for inliers}\label{subsec:outlier_adapt}
One limitation of many previous studies for both pairwise and multi-graph matching as mentioned in this paper, is that they enforce all nodes in one graph to have matchings from the other(s), either by explicitly using a permutation matrix imposing strict node-to-node correspondences \cite{UmeyamaPAMI88,ZhouCVPR12}, or implicitly doing so by generating node-to-node matchings as many as possible such that the objective affinity score is optimized \cite{ChoICCV13,ChertokPAMI10,DuchenneCVPR09,CaetanoPAMI09,GoldPAMI96}. A common simplification when applying a pairwise GM method is assuming one of two graphs is a reference graph, such that each of its node can find a match from the other testing graph.

We are motivated to devise a general outlier-rejection mechanism not depending on the particular characters of the testing data in terms of the node distribution, weight attributes, noises type, and graph structure \emph{etc}.

The Graduated Assignment method \cite{GoldPAMI96} attempts to handle outlier nodes by assigning them to additional slack variables. While in the presence of a majority of outliers, this approach is not widely adopted due to its sensitivity to parameters for the slack variables. The very recent work \cite{ChoCVPR14} proposes a general approach via a max-pooling mechanism suited for the scenario of matching two graphs with a majority of outliers. A less relevant work is by Torresani et al. \cite{TorresaniECCV08} which also allows nodes to be assigned in an unmatchable status in an energy function induced on the Markov Random Field. However, its complex objective function is designed that account for various similarity measurements \emph{e.g.} appearance descriptors, occlusions, spatial proximity \emph{etc.} beyond the general setting of the affinity matrix considered in this paper. Moreover, the associated parameters need to be learned with labeled correspondence ground truth from a training dataset which impedes its applicability. Other pairwise GM methods \cite{CollinsECCV14,ChoCVPR12} integrate the point detection and matching in a synergic manner by an integrated system, while this paper assumes the graphs are given and specifically confined in visual problems.

As a building-block, all these pairwise GM methods can be used by our approaches and other state-of-the-arts \cite{ChenICML14,YanICCV13,YanTIP15,PachauriNIPS13} to improve the initial $\mathbb{X}^{(0)}$. Nevertheless, the context of multiple graphs allows for a more robust mechanism for pruning outliers such that only the affinity and consistency relevant to common inliers are considered in our boosting framework.

We describe our `inlier nodes' eliciting method using the node-wise consistency $C_{n}(u^k,\mathbb{X})$ by Definition (\ref{def:node_consist}). Another alternative is using the node-wise affinity $S_{n}(u^k,\mathbb{X},\mathbb{K})$ via Definition (\ref{def:node_affinity}), which can be applied in a similar fashion thus its description is omitted here.

First, given the matching configuration $\mathbb{X}$, each node in one graph is scored by the node-wise consistency by Definition (\ref{def:node_consist}). Assume the number of common inliers \emph{i.e.} $n_i$ for all graphs is known, which is available when a reference template is given, or estimated by other means \emph{e.g.} \cite{ChenICML14}. How to estimate $n_i$ is not the focus of this paper.

Then we make two revisions for all of our methods: i) the pairwise affinity term $\text{vec}(\textbf{X})^T\textbf{K}\text{vec}(\textbf{X})$ is modified by keeping the rows of $\textbf{X}$ that correspond to the first $n_i$ (the number of inliers) nodes in descending order by their node-wise consistency score $C_{n}(u^k,\mathbb{X})$ in Definition (\ref{def:node_consist}), and zeroing the rest of rows. This is because the affinity or consistency between outliers is irrelevant to the semantic matching accuracy thus shall be excluded in the boosting procedure. For node-wise consistency, we use $\psi_c(\textbf{X},\mathbb{X},n_i)$ (or $\psi_a(\textbf{X},\mathbb{X},\mathbb{K},n_i)$ for node-wise affinity) to denote this `mask' operation on $\textbf{X}$ as it is determined by both the input configuration $\mathbb{X}$ and $n_i$ (also $\mathbb{K}$ in case of the affinity metric). Then the affinity score is rewritten as $J^{\psi_c}=\text{vec}(\psi_c(\textbf{X},\mathbb{X},n_i))^T\textbf{K}\text{vec}(\psi_c(\textbf{X},\mathbb{X},n_i))$; ii) the consistency term is modified by the following two equations for unary and pairwise consistency:
\begin{footnotesize}
\begin{align}\label{eq:elicit_graph_wise_consistency}
&C^{\psi_c}_u(k,\mathbb{X}^{(t-1)},n_i)\\\notag
=&1-\frac{\sum_{i=1}^{N-1}\sum_{j=i+1}^{N}\|\psi_c(\textbf{X}^{(t-1)}_{ij}-\textbf{X}^{(t-1)}_{ik}\textbf{X}^{(t-1)}_{kj},\mathbb{X}^{(t-1)},n_i)\|_F}{n_iN(N-1)}\\ \label{eq:elicit_pairwise_consistency}&C^{\psi_c}_{p}(\textbf{X}^{(t-1)}_{ij},\mathbb{X}^{(t-1)},n_i)\\\notag
=&1-\frac{\sum_{k=1}^{N}\|\psi_c(\textbf{X}^{(t-1)}_{ij}-\textbf{X}^{(t-1)}_{ik}\textbf{X}^{(t-1)}_{kj},\mathbb{X}^{(t-1)},n_i)\|_F}{2n_iN}
\end{align}
\end{footnotesize}
We use the above tailored variants for affinity score and consistency to replace the original $J$, $C_u$ and $C_{p}$ in the evaluation functions: Eq.\ref{eq:score_max_compact}, Eq.\ref{eq:score_consistency_max}, Eq.\ref{eq:score_graph_consistency_max}, Eq.\ref{eq:score_pair_consistency_estimator_max}. Similar steps are performed for using the node-wise affinity.
\section{Experiments and Discussion}
The experiments are performed on both synthetic and real-image data. The synthetic test is controlled by varying the noise level of deformation, outlier, edge density and initial pairwise matching coverage. The real images are tested by varying viewing angles, scales, shapes \emph{etc.}. The matching accuracy over all graphs, is calculated by averaging all pairwise matching accuracy \begin{small}$\frac{\sum_{i=1}^{N-1}\sum_{j=i+1}^N \text{Acc}_{ij}}{N(N-1)/2}$\end{small}. Each $\text{Acc}_{ij}$ computes the matches between $\textbf{X}_{ij}^{\text{alg}}$ given by the ground truth $\textbf{X}_{ij}^{\text{tru}}$: $\text{Acc}_{ij}=\frac{\text{tr}(\textbf{X}_{ij}^{\text{alg}}\textbf{X}_{ij}^{\text{tru}})}{\text{tr}(\textbf{1}_{n_j\times n_i} \textbf{X}_{ij}^{\text{tru}})}$. In line with \cite{YanTIP15,ZhouCVPR12}, we only calculate the accuracy for common inliers and ignore the correspondences between outliers. The above protocol is widely adopted by related work such as \cite{ChoECCV10,ZhouCVPR12}.

If not otherwise specified, the parameters of our methods are universally set by: $T_0=2$, $T_{max}=6$, $\mu^{(0)}=0.2$, $\gamma=0.3$, $\beta=1.1$. Note when $t \leq T_0$, only score boosting is performed without infusing consistency regularization. As a simple speed-up trick used in our methods, for the generated candidate solutions in the evaluation formulas $\textbf{X}_{ij}$ by Eq.\ref{eq:score_max_compact} and Eq.\ref{eq:score_consistency_max}, we compute their scores after excluding those duplicate ones.
\begin{table}[t]
\center
\small
\caption{Description of the parameters for experiment settings in this paper.}
\resizebox{0.47\textwidth}{!}{
\begin{tabular}{|c|c|c|c|c|c|c|c|c|}
  \hline
  graph\#&inlier\#&outlier\#&estimated $n_{i}$&deform&density&sensitivity&coverage&iter\#\\\hline
  $N$&$n_{i}$&$n_{o}$&$n_{est}$&$\varepsilon$&$\rho$&$\sigma$&$c$&$T_{max}$\\\hline
\end{tabular}\label{tab:parameter_meaning}}
\normalsize
\end{table}
\subsection{Dataset description and affinity setting}
\textbf{Synthetic random graph matching} The random graph test follows the widely used protocol of \cite{YanICCV13,GoldPAMI96,ChoECCV10,ZhouCVPR12,CourNIPS06}. For each trial, a reference graph with $n_{i}$ nodes is created by assigning a random weight to its edge, uniformly sampled from the interval $[0,1]$. Then the `perturbed' graphs are created by adding a Gaussian deformation disturbance to the edge weight $q_{ij}^r$, which is sampled from $N(0,\varepsilon)$ i.e. $q_{ij}^p$ = $q_{ij}^r$+$N(0,\varepsilon)$ where the superscript `p' and `r' denotes for `perturb' and `reference' respectively. Each `perturbed' graph is further added by $n_{o}$ outliers, which can also be helpful to make the graphs of equal sizes when the input graphs are different sizes. Its edge density is controlled by the density parameter $\rho\in[0,1]$ via random sampling. The edge affinity is computed by $K_{ia,jb}=\exp(-\frac{(q_{ij}-q_{ab})^2}{\sigma^2})$ where $\sigma^2$ is the edge similarity sensitivity parameter. No single-node feature is used and the unary affinity $K_{ia,ia}$ is set to zero, leaving the matching score entirely to the pairwise geometric information. In addition, we control the `coverage' of the initial matching configuration $\mathbb{X}$ by parameter $c\in [0,1]$, which denotes the rate of matching pairs that are generated by the pairwise matching solver, and the rest $1-c$ portion of pairs are assigned with a randomly generated matching solution. A description of these parameters is listed in Table \ref{tab:parameter_meaning}.

\textbf{Synthetic random point set matching} The random point set matching is also explored as tested in \cite{ChoCVPR14,LeordeanuICCV05}. First, $n_{i}$ inliers $\{\textbf{p}_k\}_{k=1}^{n_i}$ are randomly generated on the 2-D plane via Gaussian distribution $N(0,1)$ as the reference point set. Then each point is copied with Gaussian noise $N(0,\varepsilon)$ to generate $N$ random point set by further adding $n_{o}$ outliers via $N(0,1)$. The edge weight is set by the Euclidean distances $q_{ij}=\|\textbf{p}_{i}-\textbf{p}_{j}\|$ for each graph. The subsequent steps are the same with the random graph matching case. This dataset is used for the synthetic testing in comparison with MPM \cite{ChoCVPR14} under a relatively large number of outliers. This is because the max-pooling method MPM is designed for geometric relation rather than arbitrary edge weights as in the random graph matching setting.

\textbf{CMU-POSE Sequence} This data contains four sequences. Two sequences are from the CMU house (30 landmarks, 101 frames), hotel (30 landmarks, 111 frames) sequence (http://vasc.ri.cmu.edu//idb/html/motion/) which are commonly used in \cite{YanICCV13,ChoECCV10,CaetanoPAMI09,ZhouCVPR12,ChoICCV13}. The other two sequences are sampled from the `VolvoC70' and the `houseblack' (both 19 landmarks, 225 frames) covering a range of 70 degrees of viewing angles from the POSE dataset \cite{VikstenICRA09}. We use this data for the `partial similarity' test. We select $n_{i}$=$10$ landmarks out of all $n_{ant}$ annotated landmarks, and randomly choose $n_{o}$=$4$ nodes from the rest $n_{ant}$$-$$n_{i}$ nodes as `outliers'. The edge sampling follows the same way as \cite{ZhouCVPR12} by constructing the sparse delaunay triangulation among the points (no distinction to inliers or outliers). The affinity matrix is set by $K_{ia,jb}=\exp(\frac{(d_{ij}-d_{ab})^2}{-\sigma^2})$ where $d_{ij}$, $d_{ab}$ are the Euclidean distances between two points normalized to $[0,1]$ by dividing the largest edge length. The diagonal is set to zero same as \cite{ChoECCV10,ZhouCVPR12}.

\textbf{WILLOW-ObjectClass} The object class dataset released in \cite{ChoICCV13} is constructed by images from Caltech-256 and PASCAL VOC2007. Each object category contains different number of images: 109 Face, 50 Duck, 66 Wine bottle, 40 Motorbike, and 40 Car images. For each image, 10 feature points are manually labeled on the target object. We further add $n_o$ random outliers detected by a SIFT detector in our outlier test. The edge sampling for the affinity matrix follows the same way as \cite{ZhouCVPR12} by constructing the sparse delaunay triangulation among the landmarks as done in the CMU-POSE sequence test. For defining the edge affinity, we follow the protocol of \cite{ZhouCVPR12,ChoICCV13} that sets the final affinity matrix re-weighted by the edge length affinity and angle affinity: $K_{ia,jb}=\beta K_{ia,jb}^{\text{len}}+(1-\beta) K_{ia,jb}^{\text{ang}}$ where $\beta \in [0,1]$ is the weighting parameter. This is because this dataset contains more ambiguities in terms of structural symmetry as pointed out in \cite{YanTIP15}, than the sequence dataset if only length information is used. The angle for each edge is computed by the absolute angle between the edge and the horizontal line as used in \cite{ZhouCVPR12}. The edge affinity and angle affinity are calculated in the same way in the CMU-POSE test. For node-wise affinity is also set to zero to avoid the point-wise feature dominate the overall matching results.

For modeling affinity, this paper considers the second-order information, yet higher-order edge affinities in the hyper-matching setting \cite{LeeCVPR11,LeordeanuICCV11} are also applicable in our methods since their affinity score can be computed.
\begin{figure*}[tb]
\centering
\setlength{\abovecaptionskip}{0pt}
\setlength{\belowcaptionskip}{-10pt}
\subfigure{\label{fig:legend}
\includegraphics[width=0.98\textwidth]{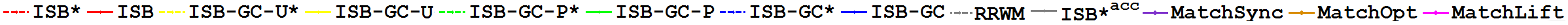}}\\\vspace{-10pt}
\setcounter{subfigure}{0}
\subfigure[Deform by $\varepsilon$]{\label{fig:random_formal_vary_deform_rrwm_acc}
\includegraphics[width=0.19\textwidth]{{{performance_random_deform_RRWM_RRWM_respect_0.08_0.18_formal_acc}}}}\hspace{-7pt}
\subfigure[Outlier by $n_o$]{\label{fig:random_formal_vary_outlier_rrwm_acc}
\includegraphics[width=0.19\textwidth]{{{performance_random_outlier_RRWM_RRWM_respect_1_6_formal_acc}}}}\hspace{-7pt}
\subfigure[Density by $\rho$]{\label{fig:random_formal_vary_density_rrwm_acc}
\includegraphics[width=0.19\textwidth]{{{performance_random_density_RRWM_RRWM_respect_0.45_0.60_formal_acc}}}}\hspace{-7pt}
\subfigure[Coverage by $c$]{\label{fig:random_formal_vary_coverage_rrwm_acc}
\includegraphics[width=0.19\textwidth]{{{performance_random_complete_RRWM_RRWM_respect_0.05_0.30_formal_acc}}}}\hspace{-7pt}
\subfigure[Outlier by $n_o$ (time)]{\label{fig:random_formal_vary_outlier_rrwm_tim}
\includegraphics[width=0.19\textwidth]{{{performance_random_outlier_RRWM_RRWM_respect_1_6_formal_tim}}}}\\\vspace{-5pt}
\subfigure[Deform by $N$]{\label{fig:random_formal_vary_graph_cnt_deform_rrwm_acc}
\includegraphics[width=0.19\textwidth]{{{performance_random_deform_RRWM_RRWM_respect_formal_all_acc}}}}\hspace{-7pt}
\subfigure[Outlier by $N$]{\label{fig:random_formal_vary_graph_cnt_outlier_rrwm_acc}
\includegraphics[width=0.19\textwidth]{{{performance_random_outlier_RRWM_RRWM_respect_formal_all_acc}}}}\hspace{-7pt}
\subfigure[Density by $N$]{\label{fig:random_formal_vary_graph_cnt_density_rrwm_acc}
\includegraphics[width=0.19\textwidth]{{{performance_random_density_RRWM_RRWM_respect_formal_all_acc}}}}\hspace{-7pt}
\subfigure[Coverage by $N$]{\label{fig:random_formal_vary_graph_cnt_coverage_rrwm_acc}
\includegraphics[width=0.19\textwidth]{{{performance_random_complete_RRWM_RRWM_respect_formal_all_acc}}}}\hspace{-7pt}
\subfigure[Deform by $N$ (time)]{\label{fig:random_formal_vary_graph_cnt_deform_rrwm_tim}
\includegraphics[width=0.19\textwidth]{{{performance_random_deform_RRWM_RRWM_respect_formal_all_tim}}}}\\
\caption{Evaluation for ISB, ISB-GC and the two derivatives ISB-GC-U and ISB-GC-P, and state-of-the-arts on the synthetic random graph dataset by varying the disturbance level (top row), and by varying the number of considered graphs (bottom row). Refer to Table.\ref{tab:case_study_setting} for the settings.}
\label{fig:formal_syn_test}
\end{figure*}
\begin{figure*}[tb]
\centering
\setlength{\abovecaptionskip}{0pt}
\setlength{\belowcaptionskip}{-10pt}
\subfigure{\label{fig:legend}
\includegraphics[width=0.98\textwidth]{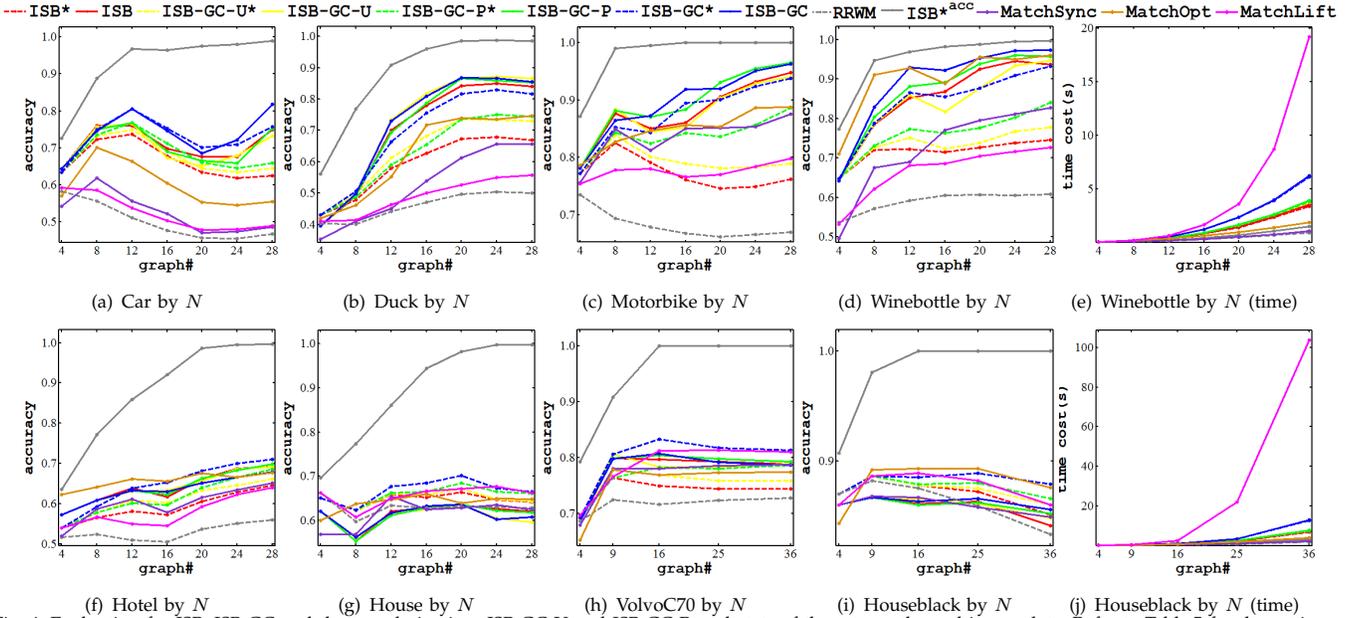}}\\\vspace{-10pt}
\setcounter{subfigure}{0}
\subfigure[Car by $N$]{\label{fig:object_formal_car_rrwm_vary_cnt_acc}
\includegraphics[width=0.19\textwidth]{{{performance_object_car_RRWM_RRWM_respect_formal_all_acc}}}}\hspace{-7pt}
\subfigure[Duck by $N$]{\label{fig:object_formal_duck_rrwm_vary_cnt_acc}
\includegraphics[width=0.19\textwidth]{{{performance_object_duck_RRWM_RRWM_respect_formal_all_acc}}}}\hspace{-7pt}
\subfigure[Motorbike by $N$]{\label{fig:object_formal_motorbike_rrwm_vary_cnt_acc}
\includegraphics[width=0.19\textwidth]{{{performance_object_motorbike_RRWM_RRWM_respect_formal_all_acc}}}}\hspace{-7pt}
\subfigure[Winebottle by $N$]{\label{fig:object_formal_winebottle_rrwm_vary_cnt_acc}
\includegraphics[width=0.19\textwidth]{{{performance_object_winebottle_RRWM_RRWM_respect_formal_all_acc}}}}\hspace{-7pt}
\subfigure[Winebottle by $N$ (time)]{\label{fig:object_formal_winebottle_rrwm_vary_cnt_tim}
\includegraphics[width=0.19\textwidth]{{{performance_object_winebottle_RRWM_RRWM_respect_formal_all_tim}}}}\\\vspace{-5pt}
\subfigure[Hotel by $N$]{\label{fig:cmu_formal_hotel_rrwm_vary_cnt_acc}
\includegraphics[width=0.19\textwidth]{{{performance_cmu_hotel_RRWM_RRWM_respect_formal_all_acc}}}}\hspace{-7pt}
\subfigure[House by $N$]{\label{fig:cmu_formal_house_rrwm_vary_cnt_acc}
\includegraphics[width=0.19\textwidth]{{{performance_cmu_house_RRWM_RRWM_respect_formal_all_acc}}}}\hspace{-7pt}
\subfigure[VolvoC70 by $N$]{\label{fig:pose_formal_volvo_rrwm_vary_cnt_acc}
\includegraphics[width=0.19\textwidth]{{{performance_pose_volvoc70_RRWM_RRWM_respect_formal_all_acc}}}}\hspace{-7pt}
\subfigure[Houseblack by $N$]{\label{fig:pose_formal_houseblack_rrwm_vary_cnt_acc}
\includegraphics[width=0.19\textwidth]{{{performance_pose_houseblack_RRWM_RRWM_respect_formal_all_acc}}}}\hspace{-7pt}
\subfigure[Houseblack by $N$ (time)]{\label{fig:pose_formal_houseblack_rrwm_vary_cnt_tim}
\includegraphics[width=0.19\textwidth]{{{performance_pose_houseblack_RRWM_RRWM_respect_formal_all_tim}}}}\\\vspace{-5pt}
\caption{Evaluation for ISB, ISB-GC and the two derivatives ISB-GC-U and ISB-GC-P, and state-of-the-arts on the real image data. Refer to Table.\ref{tab:formal_study_real_setting} for the settings.}
\label{fig:formal_real_vary_cnt_test}
\end{figure*}
\begin{figure*}[t]
\centering
\setlength{\abovecaptionskip}{0pt}
\setlength{\belowcaptionskip}{-10pt}
\subfigure{\label{fig:legend}
\includegraphics[width=0.98\textwidth]{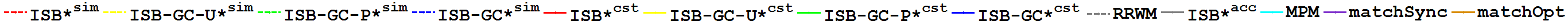}}\\\vspace{-10pt}
\setcounter{subfigure}{0}
\subfigure[Outlier by $n_{o}$]{\label{fig:random_outlier_point_rrwm_vary_nout_acc}
\includegraphics[width=0.19\textwidth]{{{performance_random_outlier_RRWM_RRWM_respect_6_16_massOutlier_acc}}}}\hspace{-7pt}
\subfigure[Outlier by $N$]{\label{fig:random_outlier_point_rrwm_vary_ngrh_acc}
\includegraphics[width=0.19\textwidth]{{{performance_random_outlier_RRWM_RRWM_respect_massOutlier_all_acc}}}}\hspace{-7pt}
\subfigure[Deform by $\varepsilon$]{\label{fig:random_outlier_point_rrwm_vary_deform_acc}
\includegraphics[width=0.19\textwidth]{{{performance_random_deform_RRWM_RRWM_respect_0.00_0.10_massOutlier_acc}}}}\hspace{-7pt}
\subfigure[Outlier by $n_{est}$]{\label{fig:random_outlier_point_rrwm_vary_nest_acc}
\includegraphics[width=0.19\textwidth]{{{performance_random_outlier_RRWM_RRWM_respect_1_18_massOutlier_spec_acc}}}}\hspace{-7pt}
\subfigure[Outlier by $n_{o}$ (time)]{\label{fig:random_outlier_point_rrwm_vary_nout_tim}
\includegraphics[width=0.19\textwidth]{{{performance_random_outlier_RRWM_RRWM_respect_6_16_massOutlier_tim}}}}\\\vspace{-5pt}
\caption{`Outlier' test for our methods driven by the consistency (solid) and affinity (dashed) inlier eliciting mechanism on random points, controlled by Table.\ref{tab:outlier_syn_setting}.}
\label{fig:outlier_random_test}
\end{figure*}
\begin{figure*}[tb]
\centering
\setlength{\abovecaptionskip}{0pt}
\setlength{\belowcaptionskip}{-10pt}
\subfigure{\label{fig:legend}
\includegraphics[width=0.98\textwidth]{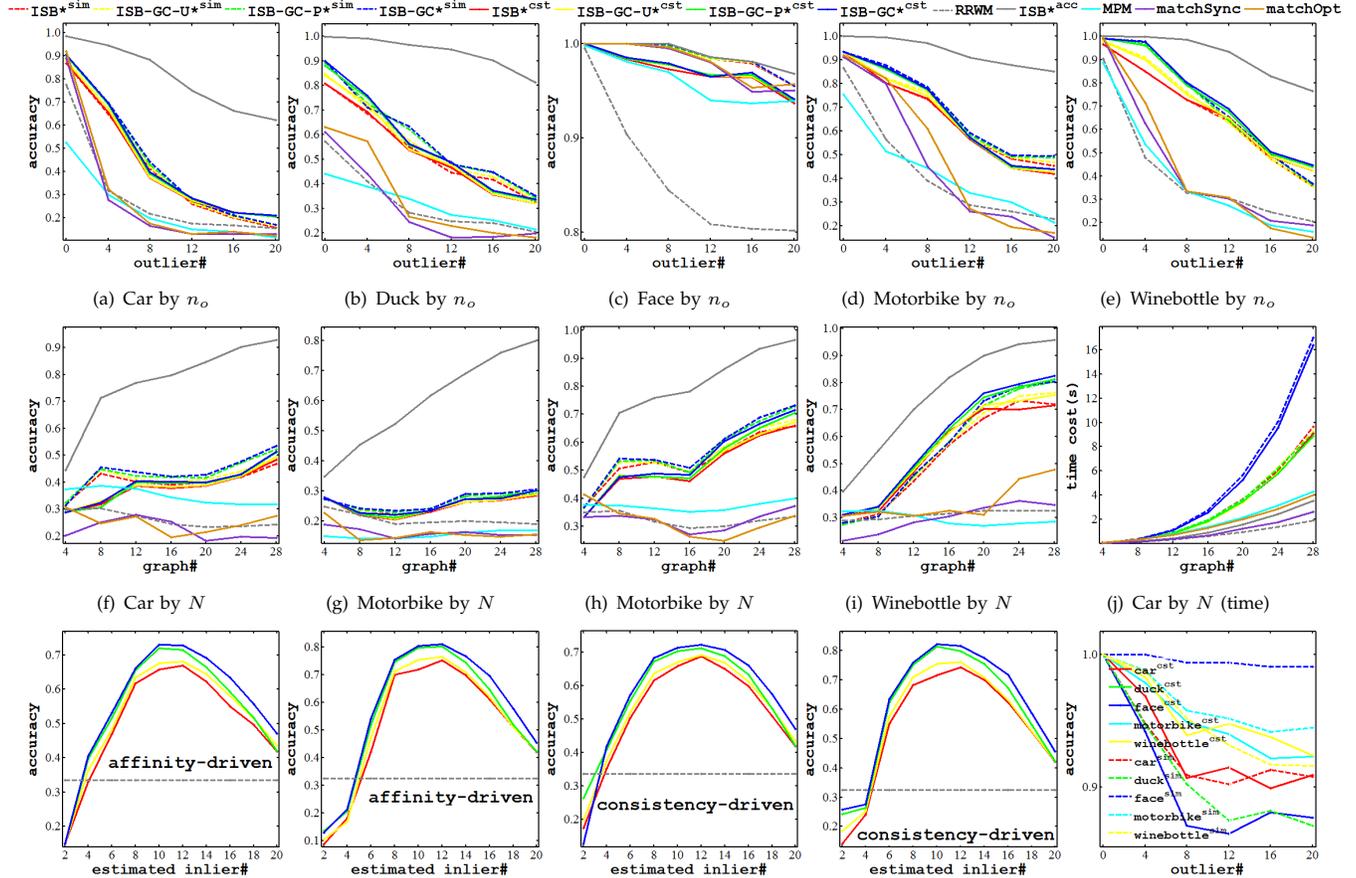}}\\\vspace{-10pt}
\setcounter{subfigure}{0}
\subfigure[Car by $n_{o}$]{\label{fig:object_outlier_car_rrwm_vary_nout_acc}
\includegraphics[width=0.19\textwidth]{{{performance_object_duck_RRWM_RRWM_respect_0_20_massOutlier_acc}}}}\hspace{-7pt}
\subfigure[Duck by $n_{o}$]{\label{fig:object_outlier_duck_rrwm_vary_nout_acc}
\includegraphics[width=0.19\textwidth]{{{performance_object_car_RRWM_RRWM_respect_0_20_massOutlier_acc}}}}\hspace{-7pt}
\subfigure[Face by $n_{o}$]{\label{fig:object_outlier_face_rrwm_vary_nout_acc}
\includegraphics[width=0.19\textwidth]{{{performance_object_face_RRWM_RRWM_respect_0_20_massOutlier_acc}}}}\hspace{-7pt}
\subfigure[Motorbike by $n_{o}$]{\label{fig:object_outlier_motorbike_rrwm_vary_nout_acc}
\includegraphics[width=0.19\textwidth]{{{performance_object_motorbike_RRWM_RRWM_respect_0_20_massOutlier_acc}}}}\hspace{-7pt}
\subfigure[Winebottle by $n_{o}$]{\label{fig:object_outlier_winebottle_rrwm_vary_nout_acc}
\includegraphics[width=0.19\textwidth]{{{performance_object_winebottle_RRWM_RRWM_respect_0_20_massOutlier_acc}}}}\\\vspace{-5pt}
\subfigure[Car by $N$]{\label{fig:object_outlier_car_rrwm_vary_ngrh_acc}
\includegraphics[width=0.19\textwidth]{{{performance_object_car_RRWM_RRWM_respect_massOutlier_all_acc}}}}\hspace{-7pt}
\subfigure[Motorbike by $N$]{\label{fig:object_outlier_duck_rrwm_vary_ngrh_acc}
\includegraphics[width=0.19\textwidth]{{{performance_object_duck_RRWM_RRWM_respect_massOutlier_all_acc}}}}\hspace{-7pt}
\subfigure[Motorbike by $N$]{\label{fig:object_outlier_motorbike_rrwm_vary_ngrh_acc}
\includegraphics[width=0.19\textwidth]{{{performance_object_motorbike_RRWM_RRWM_respect_massOutlier_all_acc}}}}\hspace{-7pt}
\subfigure[Winebottle by $N$]{\label{fig:object_outlier_winebottle_rrwm_vary_ngrh_acc}
\includegraphics[width=0.19\textwidth]{{{performance_object_winebottle_RRWM_RRWM_respect_massOutlier_all_acc}}}}\hspace{-7pt}
\subfigure[Car by $N$ (time)]{\label{fig:object_outlier_car_rrwm_vary_ngrh_tim}
\includegraphics[width=0.19\textwidth]{{{performance_object_car_RRWM_RRWM_respect_massOutlier_all_tim}}}}\vspace{-5pt}
\subfigure[Motorbike by $n_{est}$]{\label{fig:object_outlier_motorbike_rrwm_vary_estin_acc_afy}
\includegraphics[width=0.19\textwidth]{{{performance_object_motorbike_RRWM_RRWM_respect_2_20_massOutlier_spec_acc_afy}}}}\hspace{-7pt}
\subfigure[Winebottle by $n_{est}$]{\label{fig:object_outlier_winebottle_rrwm_vary_estin_acc_afy}
\includegraphics[width=0.19\textwidth]{{{performance_object_winebottle_RRWM_RRWM_respect_2_20_massOutlier_spec_acc_afy}}}}\hspace{-7pt}
\subfigure[Motorbike by $n_{est}$]{\label{fig:object_outlier_motorbike_rrwm_vary_estin_acc}
\includegraphics[width=0.19\textwidth]{{{performance_object_motorbike_RRWM_RRWM_respect_2_20_massOutlier_spec_acc}}}}\hspace{-7pt}
\subfigure[Winebottle by $n_{est}$]{\label{fig:object_outlier_winebottle_rrwm_vary_estin_acc}
\includegraphics[width=0.19\textwidth]{{{performance_object_winebottle_RRWM_RRWM_respect_2_20_massOutlier_spec_acc}}}}\hspace{-7pt}
\subfigure[Inlier hitting rate by $n_{o}$]{\label{fig:object_outlier_rrwm_nout_hit_rate}
\includegraphics[width=0.19\textwidth]{{{outlier_est_acc_noutlier=0-20_car_duck_face_bike_bottle_cst+afy_acc}}}}\\\vspace{-5pt}
\caption{`Outlier' test for our methods driven by the consistency (solid) and affinity (dashed) inlier eliciting mechanism on images. Refer to Table.\ref{tab:outlier_real_setting} for the settings.}
\label{fig:outlier_object_test}
\end{figure*}
\subsection{Comparing methods and time complexity}
\begin{table}[tb]
\center
\small
\caption{Time complexity comparison of the state-of-the-arts.}
\resizebox{0.47\textwidth}{!}{
\begin{tabular}{|rr|}
  \hline
 algorithm&time complexity\\\hline
 ISB (Alg.\ref{alg:ISB}), ISB-GC-U, ISB-GC-P (Alg.\ref{alg:ISB-EGC})&$O(N^3n^3+N^2\tau_{\text{pair}})$\\
 ISB-GC (Alg.\ref{alg:ISB-GC})&$O(N^4n+N^3n^3+N^2\tau_{\text{pair}})$\\
 MatchLift \cite{ChenICML14}&$O(N^3n^3+N^2\tau_{\text{pair}})$\\
 MatchOpt (non-factorized) \cite{YanICCV13,YanTIP15}&$O(N^2n^4+N^3n+N^2\tau_{\text{pair}})$\\
 MatchSync \cite{PachauriNIPS13}&$O(N^2n^3+N^2\tau_{\text{pair}})$\\
 \hline
\end{tabular}}
\label{tab:time_complexity}
\vspace{-20pt}
\end{table}

The implementations of all comparing methods are the authors' Matlab code and all tests run on a laptop (2.9G Intel Core I7 and 8G memory) with a single thread.

\textbf{Re-weighted Random Walk Matching (RRWM)} Since many existing multi-graph matching methods \cite{YanICCV13,YanTIP15,PachauriNIPS13,ChenICML14,RibaltaCVIU11,RibaltaCIARP09} require a pairwise GM solver in different ways, we choose RRWM \cite{ChoECCV10} due to its cost-effectiveness. We set its parameters $\alpha$=0.2, $\beta$=30.

\textbf{Max-Pooling Matching (MPM)} MPM \cite{ChoCVPR14} is an outlier-tolerant method computing the affinity score of each candidate match via maximal support from nearby matches. This method is tested in the presence of more outliers.

Several state-of-the-art multi-graph matching methods are evaluated together with our methods\footnote{Due to space limitation, we have to omit a few peer multi-graph/point-set matching methods, including graduated assignment based common labeling \cite{RibaltaSSPR10,RibaltaIJPRAI13}, the joint feature point matching \cite{JiaROML14} and the recent work \cite{HuangSGP13}. The method in \cite{RibaltaSSPR10,RibaltaIJPRAI13} has been evaluated in \cite{YanTIP15}, and \cite{ChenICML14} is a more advanced method originated from \cite{HuangSGP13}. The feature matching method \cite{JiaROML14} only considers point-wise features descriptors rather than the 2nd/higher-order information as formulated in the GM problem. Thus in fact this becomes a first-order point matching problem rather than graph matching.}. It shall be noted that all methods start with an initial matching configuration $\mathbb{X}^{(0)}$ obtained from RRWM in this paper.

\textbf{Alternating optimization for multi-graph matching (MatchOpt)} MatchOpt \cite{YanICCV13,YanTIP15} transforms the multi-graph matching problem into a pairwise one by a star structure, whose root is the reference graph $\mathcal{G}_r$, and the leaf are the others $\{\mathcal{G}_k\}_{k=1,\neq r}^N$. Then it sequentially updates each $\textbf{X}_{ru}$ by fixing the rest $N-2$ basis variables. Both factorized and non-factorized models regarding the affinity matrix are devised in \cite{YanTIP15}. Only the non-factorized model MatchOpt is compared as it has been shown in \cite{YanTIP15} more cost-effective. We set its iteration threshold $T_{max}=4$ as it often converges quickly.

\textbf{Match lifting via convex relaxation (MatchLift)} MatchLift \cite{ChenICML14} adopts a first-order approximate algorithm for semi-definite programming. The main cost is computing the eigenvalues for iterative updating the dual variable at cost of $O(N^3n^3)$. This solver typically requires more iterations to reach optimum results (10--20 rounds) thus we set $T_{max}=20$.

\textbf{Permutation synchronization (MatchSync)} MatchSync \cite{PachauriNIPS13} uses spectral analysis to the grouped `target matrix' $\mathcal{T}$ of size $nN\times nN$ (Eq.7 in \cite{PachauriNIPS13}) stacked by all initial $\{\textbf{X}_{ij}^{(0)}\}_{i,j=1}^N$. It involves a one-shot SVD to find the $n$ leading eigenvectors of `$\mathcal{T}$' whose cost is $O(n^3N^2)$, and $N$ iterations of Hungarian method with cost $O(Nn^3)$. Thus the total cost except the initial pairwise matching step is $O(N^2n^3)$. Since MatchSync is applicable only when $n\leq N$, thus the Maximum Spanning Tree (MST) is used on the consistency-wise super graph $\mathcal{G}^c_{sup}$ if $n> N$.

We consider the time complexity of our methods in several aspects as follows. The overall time complexity for all compared methods is summarized in Table \ref{tab:time_complexity}.

First, computing the pairwise score $J(\textbf{X}_{ik}\textbf{X}_{kj})$ in all our methods by trying $N-2$ anchor graph $\{\mathcal{G}_k\}_{k=1,\neq i,j}^N$ is repeated for $\frac{N(N-1)}{2}$ pairs, at the unit cost of $O(n^3)$ as $\textbf{X}_{ij}$ is a sparse permutation matrix. Thus the related overhead is $O(N^3n^3)$, which is also the complexity of ISB, except for computing the initial $\mathbb{X}^{(0)}$, whose complexity $O(N^2\tau_{\text{pair}})$ depends on the pairwise matching solver.

The second category refers to the computing of the pairwise consistency $\{C_p(\textbf{X}_{ik}\textbf{X}_{kj},\mathbb{X})\}_{i,j,k=1}^N$ in Algorithm ISB-GC. The complexity of $C_p$ is $O(Nn)$ by code level optimization, therefore similar to the first category, the total overhead in iteration $t$ is $O(N^4n)$, and the overall complexity of ISB-GC is $O(N^4n)+O(N^3n^3)+O(N^2\tau_{\text{pair}})$.

The third category refers to the off-line computing unary and pairwise consistency as used in ISB-GC-U/P. The cost for computing each unary consistency $C_u$ can be reduced to $O(N^2n)$, thus the overall cost of computing $\{C_u(k,\mathbb{X})\}_{k=1}^N$ for $\frac{N(N-1)}{2}$ pairs of graphs is $O(N^3n)$. As a result, the overall time complexity of ISB-GC-U is $O(N^3n^3)+O(N^3n)+O(N^2\tau_{\text{pair}})$. On the other hand, computing the off-line pairwise consistency $\{C_{p}(\textbf{X}_{ij}^{(t-1)},\mathbb{X}^{(t-1)})\}_{i=1,j=i+1}^{N-1,N}$ will cost $O(N^3n)$ at the unit expense of $O(Nn)$ for each $C_p$. The overall complexity of ISB-GC-P is also $O(N^3n^3)+O(N^3n)+O(N^2\tau_{\text{pair}})$.
\begin{table}[tb!]
\vspace{-5pt}
\center
\small
\caption{Parameter settings for real-image test in Fig.\ref{fig:formal_real_vary_cnt_test}.}
\resizebox{0.47\textwidth}{!}{
\begin{tabular}{|rrr|}
  \hline
  object&x-axis&fixed parameters\\\hline
    car, duck, bike, bottle&$N$=4-28&$c=1,\sigma^2=.1,\beta=.9,n_{i}=10,n_{o}=2$\\\hline
    hotel, house, volvoC70, houseblack&$N$=4-36&$c=1,\sigma^2=.05,\beta=0,n_{i}=10,n_{o}=4$\\
  \hline
\end{tabular}\label{tab:formal_study_real_setting}}
\normalsize
\vspace{-15pt}
\end{table}

Finally, we consider the run-time overhead for the inlier-eliciting mask operation $\psi_c(\cdot,\mathbb{X},n_i)$ or $\psi_a(\cdot,\mathbb{X},\mathbb{K},n_i)$. For $\psi_c(\cdot,\mathbb{X},n_i)$, it is in fact concerning with computing the node-wise consistency $\{C_n(u^k,\mathbb{X})\}_{u^k=1,k=1}^{n,N}$ for each node in every graph, thus the total overhead in each iteration is $O(N^3n)$ at the unit cost of $O(N^2n)$ per graph. For $\psi_a(\cdot,\mathbb{X},\mathbb{K},n_i)$, the node-wise affinity by Definition (\ref{def:node_affinity}) is $O(N^2n^3)$. Regardless $\psi_c$ or $\psi_a$ is used, the associated overhead can be absorbed in the overall complexity as shown in Table.\ref{tab:time_complexity}. We will evaluate these two inlier-eliciting metrics in the experiments.
\subsection{Experimental results and discussion}
For the `RRWM' and `MPM' curves in the plots, they are generated by applying the two methods on each pair of graphs independently. `ISB$^*$' denotes performing ISB without running the post step L\ref{alg:isl:enforce_consistency_begin}-\ref{alg:isl:enforce_consistency_end} in the chart of Alg.\ref{alg:ISB} and so for other methods. Fifty random trials are performed for all synthetic tests that generate the average results as plot in the figures. For real image data, 20 random trials are sampled from the image set for each setting. For clarity, standard deviations are not plot and they are found relatively homogenous over curves.
\begin{table}[!]
\vspace{-5pt}
\center
\small
\caption{Settings for random point set test with more outliers in Fig.\ref{fig:outlier_random_test}.}
\resizebox{0.47\textwidth}{!}{
\begin{tabular}{|rrr|}
  \hline
  x-axis&fixed parameters&results\\\hline
  $n_{o}$=6-16&$N$=20,$\varepsilon$=.02,$n_{i}$=6,$\rho$=1,$\sigma$=.05,$c$=1,$T_{max}$=6&Fig.\ref{fig:random_outlier_point_rrwm_vary_nout_acc}\\
  $N$=4-32&$\varepsilon$=.05,$n_{i}$=6,$n_{o}$=12,$\rho$=1,$\sigma$=.05,$c$=1,$T_{max}$=6&Fig.\ref{fig:random_outlier_point_rrwm_vary_ngrh_acc}\\
  $\varepsilon$=0-.1&$N$=20,$n_{i}$=6,$n_{o}$=12,$\rho$=.9,$\sigma$=.05,$c$=1,$T_{max}$=6&Fig.\ref{fig:random_outlier_point_rrwm_vary_deform_acc}\\
  $n_{est}$=1-18&$N$=20,$n_{i}$=6,$n_{o}$=12,$\varepsilon$=.05,$\rho$=1,$\sigma$=.05,$c$=1,$T_{max}$=6&Fig.\ref{fig:random_outlier_point_rrwm_vary_nest_acc}\\
  \hline
\end{tabular}\label{tab:outlier_syn_setting}}
\normalsize
\end{table}
\begin{table}[!]
\vspace{-5pt}
\center
\small
\caption{Settings of Willow-ObjectClass test with more outliers in Fig.\ref{fig:outlier_object_test}.}
\resizebox{0.47\textwidth}{!}{
\begin{tabular}{|rrr|}
  \hline
  object&x-axis&fixed parameters\\\hline
  car, duck, face, bike, bottle&$n_{o}$=0-20&$c=1,\sigma^2=.1,\beta=.9,n_{est}=n_{i}=10,N=28$\\\hline
  car, bike, bottle&$N$=4-28&$c=1,\sigma^2=.1,\beta=.9,n_{est}=n_{i}=10,n_{o}=10$\\\hline
  car, duck, bike, bottle&$n_{est}$=2-20&$c=1,\sigma^2=.1,\beta=.9,n_{i}=10,n_{o}=10,N=28$\\\hline
\end{tabular}\label{tab:outlier_real_setting}}
\normalsize
\end{table}

\textbf{In case of few outliers} Fig.\ref{fig:case_study} gives a comparison with controlled noise level, for several putative methods in this paper besides the main algorithms: \{ISB, ISB$^{\text{cst}}$, ISB$^{\text{2nd}}$\} in Sec.\ref{subsec:affinity_boosting} and \{ISB-GC, ISB-GC$^{\text{inv}}$\} in Sec.\ref{subsec:consistency_boosting}. Fig.\ref{fig:iter_study} shows the accuracy boosting behavior as a function of $T_{max}$ and based on this plot, we set $T_{max}=6$. There is little noise regarding with the affinity function for the coverage test in Fig.\ref{fig:random_iter_coverage_rrwm_acc}, thus pure affinity-driven approaches outperform. In other cases, ISB and ISB$^{\text{2nd}}$ under-perform and the latter even slightly degenerates as $T_{max}$ grows. Fig.\ref{fig:random_sample_rate_rrwm_acc} suggests the relation between the accuracy and the sampling rate $r$ for the $(N-2)*r$ graphs used from all graphs over the exhaustive searching space, to find the best anchor graph. Since our framework is applicable to other random search strategies, which might also depend on the specific applications, thus here we do not dwell on how to design an efficient random search technique. Our comments are as follows:

i) The proposed main method ISB-GC outperforms the baseline ISB and state-of-the-arts in most cases for both synthetic random graphs (Fig.\ref{fig:formal_syn_test}) and real images (Fig.\ref{fig:formal_real_vary_cnt_test}), except for the `coverage' test (Fig.\ref{fig:random_formal_vary_coverage_rrwm_acc}, Fig.\ref{fig:random_formal_vary_graph_cnt_coverage_rrwm_acc}), in which many pairwise matchings in the initial $\mathbb{X}^{(0)}$ are randomly generated and a consistent $\mathbb{X}^*$ is biased to the true accuracy. Specifically, the methods utilizing affinity information including \cite{YanTIP15} and ours, outperform MatchLift \cite{ChenICML14} and MatchSync \cite{PachauriNIPS13} that only use the initial matching configuration $\mathbb{X}^{(0)}$ as input, when $\mathbb{X}^{(0)}$ is largely corrupted as controlled by coverage rate $c$.

ii) As shown in the bottom row of Fig.\ref{fig:formal_syn_test} in addition with Fig.\ref{fig:formal_real_vary_cnt_test}, the accuracy in general increases as the number of graphs $N$ grows, though there is some fluctuations especially for the test on `car' as shown in Fig.\ref{fig:object_formal_car_rrwm_vary_cnt_acc}. We think this is due to the inherent matching difficulty of the sampled images, which is evidenced by the baseline `RRWM' whose performance also fluctuates. We find this phenomenon is more pronounced as shown in the bottom row of Fig.\ref{fig:formal_real_vary_cnt_test} due to the larger viewing range when more frames are sampled. In fact the \emph{relative} accuracy improvement against the RRWM baseline is increasing as $N$ grows. This is in line with the intuition that more graphs can help dismiss local ambiguity.

iii) For the `partial similarity' CMU-POSE data test which is often the case in reality, as shown in the bottom row of Fig.\ref{fig:formal_real_vary_cnt_test}, without enforcing full consistency (dashed curve) is more effective than using this constraint (solid curve). It suggests the efficacy of our graduated consistency regularized approach against the two-step hard synchronization methodology from another perspective.

\textbf{In case of more outliers} Our methods are tailored to the presence of a majority of outliers as described in Sec.\ref{subsec:outlier_adapt}. The performance is depicted in Fig.\ref{fig:outlier_random_test} and Fig.\ref{fig:outlier_object_test}, for the tests on synthetic point sets (inline with the setting for the compared method MPM \cite{ChoCVPR14}) and images, by varying the number of outliers, graphs and the estimated inliers\footnote{MatchLift is not plot in Fig.\ref{fig:outlier_random_test}, Fig.\ref{fig:outlier_object_test} as they are already very busy. It is not tailored for the presence of a majority of outliers thus under-perform as shown in the supplemental material. And its cost is heavier as shown in Fig.\ref{fig:random_formal_vary_outlier_rrwm_tim}, Fig.\ref{fig:random_formal_vary_graph_cnt_deform_rrwm_tim} as it takes more rounds of iterations to reach a satisfactory solution due to the first-order iterative procedure.}. No post-synchronization is performed in the outlier tests. Discussions are presented as follows:

i) In Fig.\ref{fig:object_outlier_rrwm_nout_hit_rate}, we show an example of the hitting rate of the top $n_i$ nodes in descending order by the node-wise consistency (solid curve) and affinity (dashed curve) metrics against true inliers. Given such reasonable hitting rates ($>0.8$), the suite of our inlier eliciting methods, \emph{i.e.} the consistency-driven variant (marked by the superscript `cst') and the affinity-driven one (marked by `sim') in general outperform state-of-the-arts notably. Note in Fig.\ref{fig:object_outlier_rrwm_nout_hit_rate}, the hitting rate for `face' by the affinity-driven metric is robust against the number of outliers. This is because the landmarks on face form a very distinctive structure thus its affinity is more informative.

ii) The sensitive test \emph{w.r.t.} the disturbance of $n_i$ is illustrated in Fig.\ref{fig:random_outlier_point_rrwm_vary_nest_acc} and the bottom row of Fig.\ref{fig:outlier_object_test}, for synthetic data and real images respectively. The eliciting mechanism boosts the accuracy compared with our original algorithms which in fact set $n_i$ equal to $n=20$ with no discrimination to outliers. Moreover, the accuracy decline is relatively smooth around the exact $n_i=10$ in these plots which suggests the robustness of our methods given a rough estimation of $n_i$.

iii) Two run-time overhead examples are plot in Fig.\ref{fig:random_outlier_point_rrwm_vary_nout_tim} for the synthetic test and in Fig.\ref{fig:object_outlier_car_rrwm_vary_ngrh_tim} for the real image test. ISB-GC* is more costive as we find it converges slower and need more overhead in each iteration as discussed in the complexity analysis. Also, all the proposed algorithms can be parallelized more easily than other iterative methods like matchOpt, matchSync, because updating each $\textbf{X}^{(t-1)}_{ij}$ can be done independently.

iv) MPM outperforms the baseline RRWM but does not perform as robustly as our methods. We think this is because MPM is suited under moderate noises in addition with the outliers. This is often not true in realistic scenarios and our settings -- Table.\ref{tab:outlier_syn_setting} and Table.\ref{tab:outlier_real_setting}.

v) In the synthetic point data test for Fig.\ref{fig:random_outlier_point_rrwm_vary_nout_acc} and Fig.\ref{fig:random_outlier_point_rrwm_vary_ngrh_acc}, the affinity-driven inlier eliciting variant outperforms the consistency-driven one given the deformation $\varepsilon<.05$. However, as shown in Fig.\ref{fig:random_outlier_point_rrwm_vary_deform_acc} when $\varepsilon$ grows by fixing the number of outliers, the consistency-driven mechanism outperforms when $\varepsilon>.05$. This is consistent with the motivation of this paper: consistency becomes helpful in the existence of large noises given few outliers.
\section{Conclusion}\label{sec:conclusion}
In the paper, we propose general algorithms towards multi-graph matching by incorporating both affinity score and matching consistency in an iterative approximate boosting procedure. The outlier-tolerance variants are designed by eliciting the affinity and consistency associated with inliers. The experimental results suggest that i) the method Alg.\ref{alg:ISB-GC} (ISB-GC) in general achieves more competitive accuracy compared with state-of-the-arts; ii) the methods Alg.\ref{alg:ISB-EGC} (ISB-GC-U/P) in general improve the cost-effectiveness of Alg.\ref{alg:ISB} (ISB) especially on the real image data under arbitrary noises.

{\scriptsize

\begin{thebibliography}{10}
\bibitem{YanECCV14}
J.~C. Yan, Y.~Li, W.~Liu, H.~Y. Zha, X.~K. Yang, and S.~M. Chu, ``Graduated
  consistency-regularized optimization for multi-graph matching,'' in
  \emph{ECCV}, 2014.

\bibitem{ChenICML14}
Y.~Chen, L.~Guibas, and Q.~Huang, ``Near-optimal joint object matching via
  convex relaxation,'' in \emph{ICML}, 2014.

\bibitem{ChoCVPR14}
M.~Cho, J.~Sun, O.~Duchenne, and J.~Ponce, ``Finding matches in a haystack: a
  max-pooling strategy for graph matching in the presence of outliers,'' in
  \emph{CVPR}, 2014.

\bibitem{YanTIP15}
J.~C. Yan, J.~Wang, H.~Y. Zha, and X.~K. Yang, ``Consistency-driven alternating
  optimization for multi-graph matching: A unified approach,'' \emph{to apperar
  in IEEE TIP}, 2015.

\bibitem{ConteIJPRAI04}
D.~Conte, P.~Foggia, C.~Sansone, and M.~Vento, ``Thirty years of graph matching
  in pattern recognition,'' \emph{IJPRAI}, 2004.

\bibitem{FoggiaIJPRAI14}
P.~Foggia, G.~Percannella, and M.~Vento, ``Graph matching and learning in
  pattern recognition in the last 10 years,'' \emph{IJPRAI}, 2014.

\bibitem{ZaslavskiyBio09}
M.~Zaslavskiy, F.~Bach, and J.-P. Vert, ``Global alignment of protein--protein
  interaction networks by graph matching methods,'' \emph{Bioinformatics},
  2009.

\bibitem{WilliamsPRL97}
M.~L. Williams, R.~C. Wilson, and E.~Hancock, ``Multiple graph matching with
  bayesian inference,'' \emph{PRL}, 1997.

\bibitem{HuangSGP13}
Q.~Huang and L.~Guibas, ``Consistent shape maps via semidefinite programming,''
  in \emph{Proc. Eurographics Symposium on Geometry Processing (SGP)}, 2013.

\bibitem{LiuVCG14}
S.~Liu, X.~Wang, J.~Chen, J.~Zhu, and B.~Guo, ``Topicpanorama: a full picture
  of relevant topics,'' in \emph{Proceedings of IEEE VAST}, 2014.

\bibitem{FischlerCACM81}
M.~A. Fischler and R.~C. Bolles, ``Random sample consensus: A paradigm for
  model fitting with applications to image analysis and automated
  cartography,'' \emph{Commun. ACM}, 1981.

\bibitem{ZhangIJCV94}
Z.~Zhang, ``Iterative point matching for registration of free-form curves and
  surfaces,'' \emph{IJCV}, 1994.

\bibitem{ZhouCVPR12}
F.~Zhou and F.~D. Torre, ``Factorized graph matching,'' in \emph{CVPR}, 2012.

\bibitem{ChoECCV10}
M.~Cho, J.~Lee, and K.~M. Lee, ``Reweighted random walks for graph matching,''
  in \emph{ECCV}, 2010.

\bibitem{LeordeanuNIPS09}
M.~Leordeanu, M.~Hebert, and R.~Sukthankar, ``An integer projected fixed point
  method for graph matching and map inference,'' in \emph{NIPS}, 2009.

\bibitem{EppsteinSODA95}
D.~Eppstein, ``Subgraph isomorphism in planar graphs and related problems,'' in
  \emph{SODA}, 1995.

\bibitem{luks1982isomorphism}
E.~M. Luks, ``Isomorphism of graphs of bounded valence can be tested in
  polynomial time,'' \emph{Journal of Computer and System Sciences}, 1982.

\bibitem{aho1989code}
A.~V. Aho, M.~Ganapathi, and S.~W. Tjiang, ``Code generation using tree
  matching and dynamic programming,'' \emph{TOPLAS}, 1989.

\bibitem{LeordeanuIJCV12}
M.~Leordeanu, R.~Sukthankar, and M.~Hebert, ``Unsupervised learning for graph
  matching,'' \emph{Int. J. Comput. Vis.}, 2012.

\bibitem{GoldPAMI96}
S.~Gold and A.~Rangarajan, ``A graduated assignment algorithm for graph
  matching,'' \emph{IEEE Transaction on PAMI}, 1996.

\bibitem{CaetanoPAMI09}
T.~Caetano, J.~McAuley, L.~Cheng, Q.~Le, and A.~J. Smola, ``Learning graph
  matching,'' \emph{IEEE Transaction on PAMI}, 2009.

\bibitem{EgoziPAMI13}
A.~Egozi, Y.~Keller, and H.~Guterman, ``A probabilistic approach to spectral
  graph matching,'' \emph{IEEE Transactions on PAMI}, 2013.

\bibitem{KBQAP57}
T.~C. Koopmans and M.~Beckmann, ``Assignment problems and the location of
  economic activities,'' \emph{Econometrica}, pp. 53--76, 1957.

\bibitem{LawlerMS63}
E.~Lawler, ``The quadratic assignment problem,'' \emph{Management Science},
  1963.

\bibitem{FunkhouserTOG04}
T.~Funkhouser, M.~Kazhdan, P.~Shilane, P.~Min, W.~Kiefer, A.~Tal, and
  D.~Dobkin, ``Modeling by example,'' \emph{ACM TOG}, 2004.

\bibitem{RibaltaCVIU11}
A.~Sol\'{e}-Ribalta and F.~Serratosa, ``Models and algorithms for computing the
  common labelling of a set of attributed graphs,'' \emph{CVIU}, 2011.

\bibitem{HuangTOG12}
Q.~Huang, G.~Zhang, L.~Gao, S.~Hu, A.~Butscher, and L.~Guibas, ``An
  optimization approach for extracting and encoding consistent maps in a shape
  collection,'' \emph{ACM TOG}, 2012.

\bibitem{PachauriNIPS13}
D.~Pachauri, R.~Kondor, and S.~Vikas, ``Solving the multi-way matching problem
  by permutation synchronization,'' in \emph{NIPS}, 2013.

\bibitem{YanICCV13}
J.~Yan, Y.~Tian, H.~Zha, X.~Yang, Y.~Zhang, and S.~Chu, ``Joint optimization
  for consistent multiple graph matching,'' in \emph{ICCV}, 2013.

\bibitem{RibaltaIJPRAI13}
A.~Sole-Ribalta and F.~Serratosa, ``Graduated assignment algorithm for multiple
  graph matching based on a common labeling,'' \emph{IJPRAI}, 2013.

\bibitem{LeordeanuICCV11}
M.~Leordeanu, A.~Zanfir, and C.~Sminchisescu, ``Semi-supervised learning and
  optimization for hypergraph matching,'' in \emph{ICCV}, 2011.

\bibitem{ChoICCV13}
M.~Cho, K.~Alahari, and J.~Ponce, ``Learning graphs to match,'' in \emph{ICCV},
  2013.

\bibitem{HuCVPR13}
N.~Hu, R.~M. Rustamov, and L.~Guibas, ``Graph matching with anchor nodes: a
  learning approach,'' in \emph{CVPR}, 2013.

\bibitem{ChertokPAMI10}
M.~Chertok and Y.~Keller, ``Efficient high order matching,'' \emph{PAMI}, 2010.

\bibitem{DuchenneCVPR09}
O.~Duchenne, F.~Bach, I.~Kweon, and J.~Ponce, ``A tensor-based algorithm for
  high-order graph matching,'' in \emph{CVPR}, 2009.

\bibitem{LeeCVPR11}
J.~Lee, M.~Cho, and K.~M. Lee, ``Hyper-graph matching via reweighted random
  walks,'' in \emph{CVPR}, 2011.

\bibitem{ZassCVPR08}
R.~Zass and A.~Shashua, ``Probabilistic graph and hypergraph matching,'' in
  \emph{CVPR}, 2008.

\bibitem{GoldJANN96}
S.~Gold and A.~Rangarajan, ``Softmax to softassign: neural network algorithms
  for combinatorial optimization,'' \emph{J. Artif. Neural Netw}, 1996.

\bibitem{LeordeanuICCV05}
M.~Leordeanu and M.~Hebert, ``A spectral technique for correspondence problems
  using pairwise constraints,'' in \emph{ICCV}, 2005.

\bibitem{CourNIPS06}
P.~S. T.~Cour and J.~Shi, ``Balanced graph matching,'' in \emph{NIPS}, 2006.

\bibitem{RibaltaCIARP09}
A.~Sol\'{e}-Ribalta and F.~Serratosa, ``On the computation of the common
  labelling of a set of attributed graphs,'' in \emph{CIARP}, 2009.

\bibitem{ChoCVPR12}
M.~Cho and K.~M. Lee, ``Progressive graph matching: Making a move of graphs via
  probabilistic voting,'' in \emph{CVPR}, 2012.

\bibitem{CollinsECCV14}
T.~Collins, P.~Mesejo, and A.~Bartoli, ``An analysis of errors in graph-based
  keypoint matching and proposed solutions,'' in \emph{ECCV}, 2014.

\bibitem{GavrilJA87}
F.~Gavril, ``Generating the maximum spanning trees of a weighted graph,''
  \emph{Journal of Algorithms}, pp. "592--597", 1987.

\bibitem{RibaltaSSPR10}
A.~S. Ribalta and F.~Serratosa, ``Graduated assignment algorithm for finding
  the common labelling of a set of graphs,'' in \emph{SSPR10}, 2010.

\bibitem{CharpiatICCVW09}
G.~Charpiat, ``Learning shape metrics based on deformations and transport,'' in
  \emph{ICCV Workshops}, 2009.

\bibitem{UmeyamaPAMI88}
S.~Umeyama, ``An eigendecomposition approach to weighted graph matching
  problems,'' \emph{PAMI}, 1988.

\bibitem{TorresaniECCV08}
L.~Torresani, V.~Kolmogorov, and C.~Rother, ``Feature correspondence via graph
  matching: Models and global optimization,'' in \emph{ECCV}, 2008.

\bibitem{VikstenICRA09}
F.~Viksten, P.~Forss{\'e}n, B.~Johansson, and A.~Moe, ``Comparison of local
  image descriptors for full 6 degree-of-freedom pose estimation,'' in
  \emph{ICRA}, 2009.

\bibitem{JiaROML14}
K.~Jia, T.~H. Chan, Z.~Zeng, G.~Wang, T.~Zhang, and Y.~Ma, ``Roml: a robust
  feature correspondence approach for matching objects in a set of images,''
  \emph{arXiv}, 2014.

\end{thebibliography}

\bibliographystyle{IEEEtran}
\vspace{-40pt}
\begin{IEEEbiography}[{\includegraphics[width=1in,height=1.25in,clip,keepaspectratio]{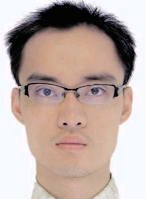}}]{Junchi Yan}
is a Ph.D. candidate at the Department of Electronic Engineering of Shanghai Jiao Tong University, China. He received the M. S. degree from the same university in 2011. He is also a Research Staff Member and IBM Master Inventor with IBM. He is the recipient of IBM Research Accomplishment and Outstanding Accomplishment Award in 2013 and 2014. His research interests are computer vision and machine learning applications, with publications including ICCV, ECCV, CVPR, IJCAI, AAAI, T-IP.
\end{IEEEbiography}
\vspace{-40pt}
\begin{IEEEbiography}[{\includegraphics[width=1in,height=1.25in,clip,keepaspectratio]{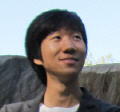}}]{Minsu Cho}
is a postdoctoral research fellow working with Prof. Jean Ponce at the WILLOW team of INRIA / Ecole Normale Superieure, Paris, France. He obtained his Ph.D. at Seoul National University, Korea. His research lies in the areas of computer vision and machine learning, especially in the problems of object recognition and graph matching.
\end{IEEEbiography}
\vspace{-40pt}
\begin{IEEEbiography}[{\includegraphics[width=1in,height=1.25in,clip,keepaspectratio]{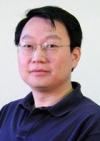}}]{Hongyuan Zha}
is a Professor at East China Normal University and the School of Computational Science and Engineering, College of Computing, Georgia Institute of Technology. He earned his PhD degree in scientific computing from Stanford University in 1993. Since then he has been working on information retrieval, machine learning applications and numerical methods. He is the recipient of the Leslie Fox Prize (1991, second prize) of the Institute of Mathematics and its Applications, the Outstanding Paper Awards of the 26th International Conference on Advances in Neural Information Processing Systems (NIPS 2013) and the Best Student Paper Award (advisor) of the 34th ACM SIGIR International Conference on Information Retrieval (SIGIR 2011). He serves as an Associate Editor of IEEE Transactions on Knowledge and Data Engineering.
\end{IEEEbiography}
\vspace{-40pt}
\begin{IEEEbiography}[{\includegraphics[width=1in,height=1.25in,clip,keepaspectratio]{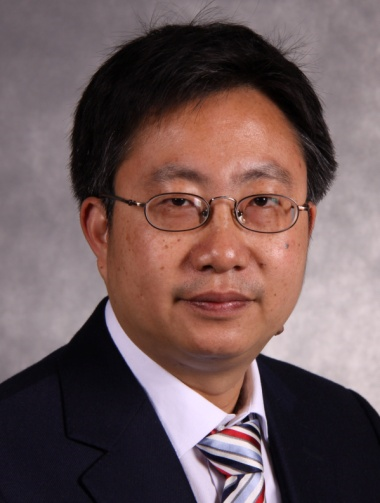}}]{Xiaokang YANG} (M'00-SM'04) received the B. S. degree from Xiamen University, Xiamen, China, in 1994, the M. S. degree from Chinese Academy of Sciences, Shanghai, China, in 1997, and the Ph.D. degree from Shanghai Jiao Tong University, Shanghai, China, in 2000. He is currently a Distinguished Professor of School of Electronic Information and Electrical Engineering, and the deputy director of the Institute of Image Communication and Information Processing, Shanghai Jiao Tong University, Shanghai, China. His research interests include visual signal processing and communication, media analysis and retrieval, and pattern recognition.
\end{IEEEbiography}
\vspace{-40pt}
\begin{IEEEbiography}[{\includegraphics[width=1in,height=1.25in,clip,keepaspectratio]{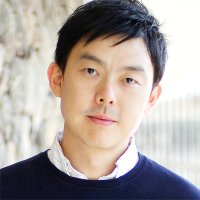}}]{Stephen Chu} is a Research Scientist and Senior Technical Staff Member at IBM. His team won the top prize in the DARPA-organized speech recognition contest for five consecutive times.  His pioneering work in Audio-Visual Speech recognition was featured in the PBS program Scientific American Frontiers with Alan Alda. His research interests cover speech recognition, computer vision, and signal processing. Stephen studied Physics at Peking University in Beijing, China, and received his M.S. and Ph.D. degrees in Electrical and Computer Engineering form University of Illinois at Urbana-Champaign.
\end{IEEEbiography}
}
\end{document}